\documentclass[letterpaper,11pt]{article}
\usepackage{etex}
\usepackage[margin=1in]{geometry}
\usepackage[dvipsnames,table,xcdraw]{xcolor}
\usepackage{amssymb,amsfonts,amsmath,amstext,amsthm,mathrsfs}
\usepackage{graphics,latexsym,epsfig,psfrag,wrapfig,comment,paralist}
\usepackage{xspace}
\usepackage{subcaption,bm,multirow}
\usepackage{booktabs}
\usepackage{mathdots}
\usepackage{bbold}
\usepackage{centernot}
\usepackage{algorithm, algorithmicx, algpseudocode}
\usepackage{prettyref}
\usepackage{listings}
\usepackage{tikz}
\usepackage{pgfplots}
\usetikzlibrary{shapes}
\usetikzlibrary{positioning, arrows, matrix, calc, patterns, fit}
\usepackage{caption}
\usepackage{array}
\usepackage{mdwmath}
\usepackage{multirow}
\usepackage{mdwtab}
\usepackage{eqparbox}
\usepackage{multicol}
\usepackage{amsfonts}
\usepackage{tikz}
\usepackage{multirow,bigstrut,threeparttable}
\usepackage{amsthm}
\usepackage{array}
\usepackage{bbm}
\usepackage{epstopdf}
\usepackage{mdwmath}
\usepackage{mdwtab}
\usepackage{eqparbox}
\usepackage{tikz}
\usepackage{latexsym}
\usepackage{amssymb}
\usepackage{bm}
\usepackage{graphicx}
\usepackage{mathrsfs}
\usepackage{epsfig}
\usepackage{psfrag}
\usepackage{setspace}
\definecolor{CeruleanBlue}{rgb}{0,0.48,0.65}
\usepackage[CJKbookmarks=true,
            bookmarksnumbered=true,
            bookmarksopen=true,
            colorlinks=true,
            citecolor=CeruleanBlue,
            linkcolor=CeruleanBlue,
            anchorcolor=red,
            urlcolor=CeruleanBlue,
            ]{hyperref}
\usepackage{natbib}
\setcitestyle{authoryear,round}
\newtheoremstyle{break}
  {\topsep}{\topsep}%
  {\itshape}{}%
  {\bfseries}{}%
  {\newline}{}%

\usepackage{pgfplots}

\pgfmathdeclarefunction{gauss}{3}{%
  \pgfmathparse{#3/(#2*sqrt(2*pi))*exp(-((x-#1)^2)/(2*#2^2))}%
}
\pgfmathdeclarefunction{mingauss}{4}{%
  \pgfmathparse{#4/(#3*sqrt(2*pi))*exp(-max((x-#1)^2, (x-#2)^2)/(2*#3^2))}%
}
\usetikzlibrary{shapes}
\usetikzlibrary{positioning, arrows, matrix, calc, patterns, fit}

\tikzset{
  invisible/.style={opacity=0},
  visible on/.style={alt={#1{}{invisible}}},
  alt/.code args={<#1>#2#3}{%
    \alt<#1>{\pgfkeysalso{#2}}{\pgfkeysalso{#3}}
  },
}

\newtheorem{definition}{\textbf{Definition}}

\newtheorem{lemma}{\textbf{Lemma}}
\newtheorem{theorem}{\textbf{Theorem}}

\newtheorem*{insight*}{\textbf{Observation}}

\newtheorem*{proposition*}{\textbf{Proposition}}

\newtheorem*{lemmai*}{\textbf{Lemma (informal)}}
\newtheorem{remark}{\textbf{Remark}}

\newtheorem{example}{\textbf{Example}}[section]

\newcommand{\cL}{\mathcal{L}}
\newcommand{\cH}{\mathcal{H}}

\newcommand{\cT}{\mathcal{T}}
\newcommand{\cC}{\mathcal{C}}

\newcommand{\cO}{\mathcal{O}}

\usepackage[normalem]{ulem}

\renewcommand{\cite}[1]{\citep{#1}}

\usepackage[intoc]{nomencl}
\makenomenclature

\makeatletter
\newcommand{\vast}{\bBigg@{2.5}}
\newcommand{\Vast}{\bBigg@{5}}
\makeatother
\newtheorem{innercustomthm}{Theorem}
\newenvironment{customthm}[1]
  {\renewcommand\theinnercustomthm{#1}\innercustomthm}
  {\endinnercustomthm}
\newtheorem{examp}{\normalfont{\textbf{Example}}}[section]
\newtheorem{lemm}{\normalfont{\textbf{Lemma}}}[section]
\newtheorem{theo}{\normalfont{\textbf{Theorem}}}[section]
\newenvironment{prooflemma}{\textit{Proof of Lemma \ref{lemma:excitation_result}.}}{\hfill$\square$}
\newenvironment{proofregret}{\textit{Proof of Theorem \ref{thm:regret_bound_paper}.}}{\hfill$\square$}
\newenvironment{proofkwidth}{\textit{Proof of Theorem \ref{thm:kwidthresult}.}}{\hfill$\square$}
\DeclareMathOperator*{\E}{\mathbb{E}}

\DeclareMathOperator*{\Prob}{\mathbb{P}}
\DeclareMathOperator*{\cov}{\text{cov}}
\DeclareMathOperator*{\tr}{\text{tr}}

\def\polylog{\operatorname{polylog}}
\def\poly{\operatorname{poly}}
\usepackage[utf8]{inputenc}
\usepackage[LGR,T1]{fontenc}

\title{SLIP: Learning to Predict in Unknown Dynamical Systems with Long-Term Memory}
\author{Paria Rashidinejad\thanks{\href{mailto:paria.rashidinejad@berkeley.edu}{\texttt{paria.rashidinejad@berkeley.edu}}} \quad Jiantao Jiao\thanks{\href{mailto:jiantao@eecs.berkeley.edu}{\texttt{jiantao@eecs.berkeley.edu}}} \quad Stuart Russell\thanks{\href{mailto:russell@cs.berkeley.edu}{\texttt{russell@cs.berkeley.edu}}}\\ { }\\
Department of Electrical Engineering and Computer Sciences\\
University of California, Berkeley}
\date{\today}

\begin{document}

\maketitle

\begin{abstract}
We present an efficient and practical (polynomial time) algorithm for online prediction in unknown and partially observed linear dynamical systems (LDS) under stochastic noise. When the system parameters are known, the optimal linear predictor is the Kalman filter.
However, the performance of existing predictive models is poor in important classes of LDS that are only marginally stable and exhibit long-term forecast memory.
We tackle this problem through bounding the generalized Kolmogorov width of the Kalman filter model by spectral methods and conducting tight convex relaxation. We provide a finite-sample analysis, showing that our algorithm competes with Kalman filter in hindsight with only logarithmic regret. Our regret analysis relies on Mendelson’s small-ball method, providing sharp error bounds without concentration, boundedness, or exponential forgetting assumptions. We also give experimental results demonstrating that our algorithm outperforms state-of-the-art methods. Our theoretical and experimental results shed light on the conditions required for efficient probably approximately correct (PAC) learning of the Kalman filter from partially observed data.

\end{abstract}

{
  \hypersetup{linkcolor=black}
  \tableofcontents
}

\newpage
\section{Introduction}
Predictive models based on linear dynamical systems (LDS) have been successfully used in a wide range of applications with a history of more than half a century. Example applications in AI-related areas range from control systems and robotics \cite{durrant2006simultaneous} to natural language processing \cite{belanger2015linear}, healthcare \Citep{parker1999model}, and computer vision \Citep{chen2011kalman, coskun2017long}. Other applications are found throughout the physical, biological, and social sciences in areas such as econometrics, ecology, and climate science.

The evolution of a discrete-time LDS is described by the following state-space model with $t\geq 1$:
\begin{align}\notag 
\begin{split}
    h_{t+1} & = Ah_t + Bx_t + \eta_t,\\
    y_t & = Ch_t + Dx_t + \zeta_t,    
\end{split}
\end{align}
where $h_t$ are the latent states, $x_t$ are the inputs, $y_t$ are the observations, and $\eta_t$ and $\zeta_t$ are process and measurement noise, respectively. 

When the system parameters are known, the optimal linear predictor is the Kalman filter. When they are unknown, a common approach for prediction is to first estimate the parameters of a Kalman filter and then use them to predict system evolution. Direct parameter estimation usually involves solving a non-convex optimization problem, such as in the expectation maximization (EM) algorithm, whose theoretical guarantees may be difficult \cite{yu2018identification}. Several recent works have studied finite-sample theoretical properties of LDS identification. 
For fully observed LDS, it has been shown that system identification is possible without a strict stability ($\rho(A) < 1$) assumption, where $\rho(A)$ is the spectral radius of $A$ \cite{simchowitz2018learning, sarkar2018near, faradonbeh2018finite}. 
For partially observed LDS, methods such as gradient descent \cite{hardt2018gradient} and subspace identification \cite{tsiamis2019finite} are developed, whose performances degrade polynomially when $\rho(A)$ is close to one.

We focus on constructing \emph{predictors} of an LDS without identifying the parameters. In the case of a stochastic LDS, the recent work of \citet{tsiamis2020online} is most related to our question. Their method performs linear regression over a fixed-length lookback window to predict the next observation $y_t$ given its causal history. Without using a mixing-time argument, \citet{tsiamis2020online} showed logarithmic regret with respect to the Kalman filter in hindsight even when the system is marginally stable ($\rho(A)\leq 1$). However, the prediction performance deteriorates if the true Kalman filter exhibits \emph{long-term forecast memory}.

To illustrate the notion of forecast memory, we recall the recursive form of the (stationary) Kalman filter for $1\leq t\leq T$, where $T$ is the final horizon
\citep[chap.~9]{kailath2000linear}:
\begin{align}
    \hat{h}_{t+1|t} & = A \hat{h}_{t|t-1} + B x_t + K (y_t - C \hat{h}_{t|t-1} - D x_t) \label{eqn.kalmanfilter} \\
    & = (A-KC) \hat{h}_{t|t-1} + Ky_t + (B-KD)x_t, \label{eqn.kalmanwithg}
\end{align}
where $\hat{h}_{t|t-1}$ denotes the optimal linear predictor of $h_t$ given all the observations $y_1,y_2,\ldots,y_{t-1}$ and inputs $x_1, x_2, \dots, x_{t-1}$. The matrix $K$ is called the (predictive) Kalman gain.\footnote{One can interpret the Kalman filter Equation~(\ref{eqn.kalmanfilter}) as linear combinations of optimal predictor given existing data $A \hat{h}_{t|t-1}$, known drift $Bx_t$, and amplified innovation $K(y_t - C \hat{h}_{t|t-1}-D x_t)$, where the term $y_t - C \hat{h}_{t|t-1}-D x_t$, called the \emph{innovation} of process $y_t$, measures how much additional information $y_t$ brings compared to the known information of observations up to $y_{t-1}$. } The Kalman predictor of $y_t$ given $y_1,y_2,\ldots,y_{t-1}$ and $x_1,x_2,\ldots,x_t$, denoted by $\hat{y}_{t|t-1}$, is $C \hat{h}_{t|t-1} + D x_t$. Assume that $\hat{h}_{1|0} = 0$. By expanding Equation~(\ref{eqn.kalmanwithg}), we obtain
\begin{align}\label{eqn.kalmanexpanded}
   m_t & \triangleq  \hat{y}_{t|t-1} = \sum\limits_{i=1}^{t-1} CG^{t-i-1}K y_i + \sum\limits_{i=1}^{t-1}  CG^{t-i-1}(B-KD)x_i + Dx_t, 
\end{align}
where $G = A-KC$. In an LDS, the transition matrix $A$ controls how fast the process mixes---i.e., how fast the marginal distribution of $y_t$ becomes independent of $y_1$. However, it is $G$ that controls how long the \emph{forecast} memory is. Indeed, it was shown in \citet[chap.~14]{kailath2000linear} that if the spectral radius $\rho(G)$ is close to one, then the performance of a linear predictor that uses only $y_{t-k}$ to $y_{t-1}$ for fixed $k$ in predicting $y_t$ would be substantially worse than that of a predictor that uses all information $y_1$ up to $y_{t-1}$ as $t\to \infty$. Conceivably, the sample size required by the algorithm of \citet{tsiamis2020online} explodes to infinity as $\rho(G)\to 1$, since the predictor uses a fixed-length lookback window to conduct linear regression. 

The primary reason to focus on long-term forecast memory is the ubiquity of long-term dependence in real applications, where it is often the case that not all state variables change according to a similar timescale\footnote{Indeed, a common practice is to set the timescale to be small enough to handle the fastest-changing variables.} \Citep{chatterjee2010dbns}. For example, in a temporal model of the cardiovascular system, arterial elasticity changes on a timescale of years, while the contraction state of the heart muscles changes on a timescale of milliseconds.

Designing provably computationally and statistically efficient algorithms in the presence of long-term forecast memory is challenging, and in some cases, impossible. A related problem studied in the literature is the prediction of auto-regressive model with order infinity: AR$(\infty)$. Without imposing structural assumptions on the coefficients of an AR$(\infty)$ model, there is no hope to guarantee vanishing prediction error. One common approach to obtain a smaller representation is to make an exponential forgetting assumption to justify finite-memory truncation. This approach has been used in approximating AR$(\infty)$ with decaying coefficients \Citep{goldenshluger2001nonasymptotic}, LDS identification \Citep{hardt2018gradient}, and designing predictive models for LDS \Citep{tsiamis2020online, kozdoba2019line}. Inevitably, the performance of these methods degrade by either losing long-term dependence information or requiring very large sample complexity as $\rho(G)$ (and sometimes, $\rho(A)$) gets closer to one.

However, the Kalman predictor in~(\ref{eqn.kalmanexpanded}) does seem to have a structure and in particular, the coefficients are geometric in $G$, which gives us hope to exploit it. Our main contributions are the following:

\textbf{1. Generalized Kolmogorov width and spectral methods:} We analyze the \emph{generalized Kolmogorov width}, defined in Section \ref{sec:generalized_kwidth}, of the Kalman filter coefficient set. In Theorem \ref{thm:kwidthresult}, we show that when the matrix $G$ is diagonalizable with \emph{real} eigenvalues, the Kalman filter coefficients can be approximated by a linear combination of $\polylog{(T)}$ \emph{fixed known} filters with $1/\poly{(T)}$ error. It then motivates the algorithm design of linear regression based on the \emph{transformed} features, where we first transform the observations $y_{1:t}$ and inputs $x_{1:t}$ for $1\leq t\leq T$ via these fixed filters. In some sense, we use the transformed features to achieve a good bias-variance trade-off: the small number of features guarantees small variance and the generalized Kolmogorov width bound guarantees small bias. We show that the fixed known filters can be computed efficiently via spectral methods. Hence, we choose spectral LDS improper predictor (SLIP) as the name for our algorithm.

\textbf{2. Difficulty of going beyond real eigenvalues:} We show in Theorem \ref{thm:kwidthresult} that if the dimension of matrix $G$ in~(\ref{eqn.kalmanexpanded}) is at least $2$, then without assuming real eigenvalues one has to use at least $\Omega(T)$ filters to approximate an arbitrary Kalman filter. In other words, the Kalman filter coefficient set is very difficult to approximate via linear subspaces in general. This suggests some inherent difficulty of constructing provable algorithms for prediction in an arbitrary LDS. 
    
\textbf{3. Logarithmic regret uniformly for $\boldsymbol{\rho(G)\leq 1, \rho(A)\leq 1}$:} When $\rho(A)$ or $\rho(G)$ is equal to one the process does not mix and common assumptions regarding boundedness, concentration, or stationarity do not hold. Recently, \citet{mendelson2014learning} showed that such assumptions are not required and learning is possible under a milder assumption referred to as the \textit{small-ball} condition. In Theorem \ref{thm:regret_bound_paper}, we leverage this idea as well as results on self-normalizing martingales and show a logarithmic regret bound for our algorithm uniformly for $\rho(G)\leq 1$ and $\rho(A)\leq 1$. A roadmap to our regret analysis method is provided in Section \ref{sec:proof_roadmap}.

\textbf{4. Experimental results:} We demonstrate in simulations that our algorithm performs better than the state-of-the-art in LDS prediction algorithms. In Section \ref{sec:experiments}, we compare the performance of our algorithm to wave filtering \cite{hazan2017learning} and truncated filtering \cite{tsiamis2020online}. 

\section{Related work}
Adaptive filtering algorithms are classical methods for predicting observations without the intermediate step of system identification \cite{ljung1978convergence, fuller1980predictors, fuller1981properties, wei1987adaptive, lai1991recursive, lorentz1996constructive}. However, finite-sample performance and regret analysis with respect to optimal filters are typically not studied in the classical literature. From a machine learning perspective, finite-sample guarantees are critical for comparing the accuracy and sample efficiency of different algorithms. 
In designing algorithms and analyses for learning from sequential data, it is common to use mixing-time arguments \cite{yu1994rates}. These arguments justify finite-memory truncation \cite{hardt2018gradient, goldenshluger2001nonasymptotic} and support generalization bounds analogous to those in i.i.d.~data \cite{mohri2009rademacher,kuznetsov2017generalization}. An obvious drawback of mixing-time arguments is that the error bounds degrade with increasing mixing time. Several recent works established that identification is possible for systems that do not mix \cite{simchowitz2018learning, faradonbeh2018finite, simchowitz2019learning}. For the problem of the linear quadratic regulator, where the state is fully observed, several results provided finite-sample regret bounds \cite{faradonbeh2017optimism, ouyang2017learning, dean2018regret, abeille2018improved, mania2019certainty,simchowitz2020naive}.

For prediction without LDS identification, \citet{hazan2017learning, hazan2018spectral} have proposed algorithms for the case of bounded adversarial noise. Similar to our work, they use spectral methods for deriving features. However, the spectral method is applied on a different set and connections with $k$-width and difficulty of approximation for the non-diagonalizable case are not studied. Moreover, the regret bounds are computed with respect to a certain fixed family of filters and competing with the Kalman filter is left as an open problem. Indeed, the predictor for general LDS proposed by \citet{hazan2018spectral} without the real eigenvalue assumption only uses a fixed lookback window. Furthermore, the feature norms are of order $\poly{(T)}$ in our formulation, which makes a naive application of online convex optimization theorems~\cite{hazan2019introduction} fail to achieve a sublinear regret.

We focus on a more challenging problem of learning to predict in the presence of unbounded stochastic noise and long-term memory, where the observation norm grows over time. The most related to our work are the recent works of \citet{tsiamis2020online} and \citet{ghai2020no}, where the performance of an algorithm based on a finite lookback window is shown to achieve logarithmic regret with respect to the Kalman filter. However, the performance of this algorithm degrades as the forecast memory increases. In fact, this algorithm can be viewed as a special case of our algorithm where the fixed filters are chosen to be standard basis vectors.

We investigate the possibility of conducting tight convex relaxation of the Kalman predictive model by defining a notion that generalizes Kolmogorov width. The Kolmogorov width is a notion from approximation theory that measures how well a set can be approximated by a low-dimensional linear subspace \cite{pinkus2012n}. Kolmogorov width has been used in a variety of problems such as minimax risk bounds for truncated series estimators \cite{donoho1990minimax, javanmard2012minimax}, minimax rates for matrix estimation \cite{ma2015volume}, density estimation \cite{hasminskii1990density}, hypothesis testing \cite{wei2020local, wei2020gauss}, and compressed sensing \cite{donoho2006compressed}. In Section \ref{sec:kolmogorov}, we present a generalization of Kolmogorov width, which facilitates measuring the convex relaxation approximation error.

\section{Preliminaries and problem formulation}

\subsection{Notation}
We denote by $x_{1:t} \in \mathbb{R}^{nt}$, the vertical concatenation of $x_1, \dots, x_t \in \mathbb{R}^n$. We use $x_t(i)$ to refer to the $i$-th element of the vector $x_t = [x_t(1), \dots, x_t(n)]^\top$. We denote by $\|.\|_2$, the Euclidean norm of vectors and the operator 2-norm of matrices. The spectral radius of a square matrix $A$ is denoted by $\rho(A)$. The eigenpairs of an $n \times n$ matrix are $\{(\sigma_j, \phi_j)\}_{j=1}^n$ where $\sigma_1 \geq \dots \geq \sigma_n$ and $\{\phi_j\}_{j=1}^k$ are called the top $k$ eigenvectors. We denote by $\phi_j(t:1) = [\phi_j(t), \dots, \phi_j(1)]$ the first $t$ elements of $\phi_j$ in a reverse order. The horizontal concatenation of matrices $a_1, \dots, a_n$ with appropriate dimensions, is denoted by $[a_i]_{i=1}^n = [a_1 | \dots | a_n]$. The Kronecker product of matrices $A$ and $B$ is denoted by $A \otimes B$. 
Identity matrix of dimension $n$ is represented by $I_n$. We write $x \lesssim_{b} y$ to represent $x \leq cy$, where $c$ is a constant that only depends on $b$. We use the notation $x \asymp_b y$ if $c_1, c_2 > 0$ exist that only depend on $b$ and $ c_1 |x| \leq |y| \leq c_2 |x|$. We define $M = (R_\Theta, m, \gamma, \kappa, \beta, \gamma, \delta)$ to be a shorthand for the PAC bound parameters (defined in Theorem \ref{thm:regret_bound_paper}). 
Given a function $f: \mathbb{N} \rightarrow R$, we write $x \lesssim_M f(T), x \asymp_M f(T)$ to specify the dependency only on the horizon $T$.

\subsection{Problem statement}\label{sec:problem_statement}

We consider the problem of predicting observations generated by the following linear dynamical system with inputs $x_t \in \mathbb{R}^n$, observations $y_t \in \mathbb{R}^m$, and latent states $h_t \in \mathbb{R}^d$: 
\begin{align}\label{eq:lds_definition}
\begin{split}
    h_{t+1} & = Ah_t + Bx_t + \eta_t,\\
    y_t & = Ch_t + Dx_t + \zeta_t,
\end{split}
\end{align}
where $A, B, C,$ and $D$ are matrices of appropriate dimensions. The sequences $\eta_t \in \mathbb{R}^d$ (process noise) and $\zeta_t \in \mathbb{R}^m$ (measurement noise) are assumed to be zero-mean, i.i.d.~random vectors with covariance matrices $Q$ and $R$, respectively. For presentation simplicity, we assume that $\eta_t$ and $\zeta_t$ are Gaussian; extension of our regret analysis to sub-Gaussian and hypercontractive noise is straightforward. We assume that the discrete Riccati equation of the Kalman filter for the state covariance has a solution $P$ and the initial state starts at this stationary covariance. This assumption ensures the existence of the stationary Kalman filter with stationary gain $K$; see \citet{kailath2000linear} for details.

Define the observation matrix $\cO_t$ and the control matrix $\mathcal{C}_t$ of a stationary Kalman filter as
\begin{align}\label{eq:O_t_and_C_t_definitions}
    \begin{split}
        \cO_t & = \begin{bmatrix} CG^{t-1}K & CG^{t-3}K & \dots & CK\end{bmatrix}, \\
        \mathcal{C}_t & = \begin{bmatrix}CG^{t-1}(B-KD) & CG^{t-3}(B-KD) & \dots & C(B-KD)\end{bmatrix}.
    \end{split}
\end{align}
where $G = A-KC$ is called the closed-loop matrix. The Kalman predictor~(\ref{eqn.kalmanexpanded}) can be written as
\begin{align}\label{eq:unrolled_kalman_predictions}
    m_{t+1} = \cO_t y_{1:t} + \mathcal{C}_t x_{1:t} + Dx_{t+1},
\end{align}
The prediction error $e_t = y_t - m_t$, also called the \textit{innovation}, is zero-mean with a stationary covariance $V$. Our goal is to design an algorithm $\hat{m}_t(y_{1:t-1},x_{1:t})$ such that the following regret
\begin{align}\label{eq:regret_definition}
    \text{Regret}(T) \triangleq \sum_{t=1}^T \|y_t - \hat{m}_t\|_2^2 - \|y_t - m_t\|_2^2 
\end{align}
is bounded by $\polylog(T)$ with high probability.

\subsection{Improper learning}
Most existing algorithms for LDS prediction include a preliminary system identification step, in which system parameters are first estimated from data, followed by the Kalman filter. However, the loss function (such as squared loss) over system parameters is non-convex, for which methods based on heuristics such as EM and subspace identification are commonly used. Instead, we aspire to an algorithm that optimizes a convex loss function for which theoretical guarantees of convergence and sample complexity analysis are possible. This motivates developing an algorithm based on \textit{improper learning.}

Instead of directly learning the model parameters in a hypothesis class $\cH$, improper learning methods reparameterize and learn over a different class $\widetilde{\cH}$. For example in system \eqref{eq:lds_definition}, proper learning hypothesis class $\cH$ contains possible values for parameters $A, B, C, D, Q$ and $R$. Improper learning is used for statistical or computational considerations when the original hypothesis class is difficult to learn. The class $\widetilde{\cH}$ is often a \textit{relaxation}: it is chosen in a way that is easier to optimize and more computationally efficient while being close to the original hypothesis class. Improper learning has been used to circumvent the proper learning lower bounds \cite{foster2018logistic}. 

In this paper, we use improper learning to conduct a tight \textit{convex relaxation}, i.e. we slightly overparameterize the LDS predictive model in such a way that the resulting loss function is convex. Designing an overparameterized improper learning class requires care as using a small number of parameters may result in a large bias whereas using too many parameters may result in high variance. Section \ref{sec:filter_approximation} presents our overparameterization approach based on spectral methods that enjoys a small approximation error with relatively few parameters.

\subsection{Systems with long forecast memory}
As discussed before, system \eqref{eq:lds_definition} exhibits long forecast memory when $\rho(G)$ is close to one. The closed-loop matrix $G$ itself is related to parameters $A, C, Q,$ and $R$. In the following example, we discuss when long forecast memory is instantiated in a scalar dynamical system.

\begin{example} Consider system \eqref{eq:lds_definition} with $d = m = 1$. The following holds for a stationary Kalman filter 
\begin{align*}
    KC = \frac{A C^2P^+  }{C^2P^+  + R} \Rightarrow 0 \leq KC \leq A \qquad \text{for } d = m = 1,
\end{align*}
where $P^+$ is the variance of state predictions $\hat{h}_{t|t-1}$ \cite{kailath2000linear}. The above constraint yields $G = A - KC \leq A$, which implies that the forecast memory can only be long in systems that mix slowly. We write
\begin{align*}
    G = A \big(1 - \frac{C^2P^+}{C^2 P^+ + R}\big), \qquad \text{for } d = m = 1.
\end{align*}
The above equation suggests if $R \gg C^2P^+$, then $G$ is close to $A$. In words, linear dynamical systems with small observed signal to noise ratio $C/\sqrt{R}$ have long forecast memory, provided that they mix slowly.

Another parameter that affects the forecast memory of a system is the process noise variance $Q$. When $Q$ is small and $A$ is close to one, latent state $h_t$ is almost constant. In this setting, the observations in the distant past are informative on $h_t$ and therefore should be considered when making predictions.
\end{example}

In multi-dimensional systems, the chance of encountering a system with long forecast memory is much higher as it suffices for only one variable or direction to exhibit long forecast memory. Systems represented in the discrete-time form of Equation \eqref{eq:lds_definition} are often obtained by discretizing differential equations and continuous dynamical systems, for which choosing a small time step results in a better approximation. However, reducing the time step directly increases the forecast memory. These types of issues has motivated a large body of research on alternative methods such as continuous models \Citep{nodelman2002continuous} and adaptive time steps \Citep{aleks2009probabilistic}. It is therefore desirable to have algorithms whose performance is not affected by the choice of time step, which is one of our goals in this paper.

\section{SLIP: Spectral LDS improper predictor}\label{sec:skf_paper}
In this section, we present the SLIP algorithm and the main regret theorem. The derivation of the algorithm and the sketch for regret analysis are respectively provided in Section \ref{sec:kolmogorov} and Section \ref{sec:proof_roadmap}.

Algorithm \ref{alg:slip} presents a pseudocode for the SLIP algorithm. Our algorithm is based on an online regularized least squares and a linear predictor $\hat{m}_t = \hat{\Theta}^{(t)}f_t$, where $f_t$ is an $l$-dimensional vector of features and $\hat{\Theta}^{(t)} \in \mathbb{R}^{m \times l}$ is a parameter matrix. The features are constructed from past observations and inputs using eigenvectors of a particular $T\times T$ Hankel matrix with entries
\begin{align}\label{eq:hankel_matrix}
    H_{ij} = \frac{1+(-1)^{i+j}}{2(i+j-1)}, \quad 1\leq i,j\leq T.
\end{align}
Let $\phi_1, \dots, \phi_k$ for $k \leq T$ be the top $k$ eigenvectors of matrix $H$, to which we refer as \textit{spectral filters}. At every time step, we obtain our feature vector by concatenating the current input $x_t$ to $k$ \textit{output features} based on $y_{1:t-1}$ and $k$ \textit{input features} based on $x_{1:t-1}$. More specifically, we have

\begin{align} \label{eq:feature_def}
        \begin{split}
        \widetilde{y}_{t-1}(j) & \triangleq (\phi^\top_j({t-1}:1) \otimes I_m) y_{1:{t-1}} = \phi_j(1) y_{t-1} + \dots +\phi_j({t-1})y_1 \quad \text{(output features)},\\
        \widetilde{x}_{t-1}(j) & \triangleq (\phi^\top_j({t-1}:1) \otimes I_n) x_{1:{t-1}} = \phi_j(1) x_{t-1} + \dots +\phi_j({t-1})x_1 \quad \text{(input features)},
        \end{split}
\end{align}
for $j \in \{1, \dots , k\}$, resulting in a feature vector $f_t$ with dimension $l = mk+nk+n$. 
Upon receiving a new observation, the parameter matrix is updated by minimizing the regularized loss 
\begin{align*}
    \sum_{i=1}^t \|\hat{\Theta} f_t - y_t\|^2 + \alpha \|\hat{\Theta}\|^2_2,
\end{align*}
for $\alpha > 0$, which yields the following update rule
\begin{align}\label{eq:theta_update_rule}
    \hat{\Theta}^{(t+1)} = \big(\sum_{i=1}^t y_i f_i^\top  \big) \big(\sum_{i=1}^t f_i f_i^\top  +\alpha I_l \big)^{-1}.
\end{align}

\begin{algorithm}[H]
\caption{SLIP: \textbf{S}pectral \textbf{L}DS \textbf{I}mproper \textbf{P}redictor}\label{alg:slip}
\begin{algorithmic}
\State \textbf{Inputs:} \hspace{0.5mm}  Time horizon $T$, number of filters $k$, regularization parameter $\alpha$, input dimension $n$, \\
\hspace{1.6cm} observation dimension $m$.
\State \textbf{Output:} One-step-ahead predictions $\hat{m}_t(x_{1:t}, y_{1:t-1})$. \\
\vspace{-0.2cm}
\State Compute the top $k$ eigenvectors $\{\phi_j\}_{j=1}^k$ of matrix $H$ with elements
\begin{align*}
    H_{ij} = \frac{(-1)^{i+j}+1}{2(i+j-1)}, \quad 1\leq i,j\leq T.
\end{align*}
\State Set vectors $\psi_i = [\phi_1(i), \dots, \phi_k(i)]^\top$ for $i \in \{1, \dots, T\}$, where $\phi_j(i)$ is the $i$-th element of $\phi_j$.

\State Initialize $\hat{\Theta}^{(1)} \in \mathbb{R}^{m \times l}$ with $l = (n+m)k +n$.
\For {$t = 1, \dots, T$}
\State Set $\Psi_{t-1} = [\psi_{t-1}, \dots, \psi_1]$, where $\Psi_0 = 0_{k}.$
\State Set $x_{1:t-1} = [x_1^\top, \dots, x_{t-1}^\top]^\top, y_{1:t-1} = [y_1^\top, \dots, y_{t-1}^\top]^\top$, $x_{1:0} = 0_n, y_{1:0} \triangleq 0_m$.
\State Compute $l$-dimensional feature vector $f_t$:\\
$\vcenter{
   \begin{align}\notag
     f_t = \begin{bmatrix}
     \widetilde{y}_{t-1}\\
     \widetilde{x}_{t-1}\\
     x_t
     \end{bmatrix} = 
     \begin{bmatrix}
     (\Psi_{t-1} \otimes I_m) y_{1:t-1}\\
     (\Psi_{t-1} \otimes I_n) x_{1:t-1}\\
     x_t
     \end{bmatrix}.
   \end{align}}$
\State Predict $\hat{m}_t = \hat{\Theta}^{(t)}f_t$.
\State Observe $y_t$ and update parameters $\hat{\Theta}^{(t+1)} = \big(\sum_{i=1}^t y_i f_i^\top  \big) \big(\sum_{i=1}^t f_i f_i^\top  +\alpha I_l \big)^{-1}$.
\EndFor
\end{algorithmic}
\end{algorithm}
Importantly, Algorithm  \ref{alg:slip} requires no knowledge of the system parameters, noise covariance, or state dimension and the predictive model is learned online only through sequences of inputs and observations. Note that the spectral filters are computed by conducting a single eigendecomposition and are fixed throughout the algorithm; matrix $\Psi_t$ merely selects certain elements of spectral filters used for constructing features. Computing eigenvectors when $T$ is large is possible by solving the corresponding second-order Sturm-Liouville equation, which allows using efficient ordinary differential equation solvers; see \citet{hazan2017learning} for details.

The next theorem analyzes the regret achieved by the SLIP algorithm. A proof sketch of the theorem is provided in Section \ref{sec:proof_roadmap} and a complete proof is deferred to Appendix \ref{app:regret_analysis}.

\begin{theorem}\label{thm:regret_bound_paper} \normalfont \textbf{(Regret of the SLIP algorithm)} \textit{Consider system \eqref{eq:lds_definition} without inputs with initial state covariance equal to the stationary covariance $P$. Let $m_t$ be the predictions made by the best linear predictor (Kalman filter) and $\hat{m}_t$ be the predictions made by Algorithm \ref{alg:slip}. Fix the failure probability $\delta > 0$ and make the following assumptions:}
\textit{\begin{enumerate}[(i)]
    \item
    There exists a finite $R_\Theta$ that $\|C\|_2,\|P\|_2, \|Q\|_2, \|R\|_2, \|V\|_2 \leq R_\Theta$ and $\|\cO_t\|_2 \leq R_\Theta t^{\beta}$ for a bounded constant $\beta \geq 0$. Let $\kappa$ be the maximum condition number of $R$ and $Q$.
    \item The system is marginally stable with $\rho(A) \leq 1$ and $\|A^t\|_2 \leq \gamma t^{\log (\gamma)}$ for a bounded constant $\gamma \geq 1$.
     Furthermore, the closed-loop matrix $G$ is diagonalizable with real eigenvalues.
    \item The regularization parameter $\alpha$ and the number of filters $k$ satisfy the following
    \begin{align*}
        k \asymp \log^2(T) \polylog (m,\gamma ,R_{\Theta},\frac{1}{\delta}), \qquad \alpha \asymp \frac{1}{R_\Theta k T^\beta}
    \end{align*}
    \item There exists $s \lesssim_{R_\Theta, m, \gamma, \beta, \delta} t/(k\log k)$ 
    and $t_0$ such that for all $t \geq t_0$
    \begin{align}\label{eq:filter_condition}
     t \Omega_{s/2}(A; \psi) - \Omega_{t+1}(A; \psi) \succeq 0.
    \end{align}
    $\Omega_t(A; \psi)$ is called the \normalfont{filter quadratic function} \textit{of $\psi$ with respect to $A$ defined as}
    \begin{align*}
        \Omega_t(A; \psi) & = (\psi^{(d)}_1) (\psi^{(d)}_1)^\top + (\psi^{(d)}_2 + \psi^{(d)}_1 A) (\psi^{(d)}_2 + \psi^{(d)}_1 A)^\top + \dots\\
        & + (\psi^{(d)}_{t-1} + \dots + \psi^{(d)}_1 A^{t-2}) (\psi^{(d)}_{t-1}  + \dots + \psi^{(d)}_1 A^{t-2})^\top
    \end{align*} 
    \textit{where $\psi^{(d)}_i = [\phi_1(i), \dots, \phi_k(i)]^\top \otimes I_d$.}
\end{enumerate}}
\textit{Then, for all $T \geq \max\{10,t_0\}$, the following holds with probability at least $1-\delta$,
}
\begin{align*}
    \text{Regret}(T) \leq \polylog(T, \gamma, \frac{1}{\delta}) \kappa \poly(R_\Theta, \beta, m).
\end{align*}
\end{theorem}
Theorem \ref{thm:regret_bound_paper} states that if $G$ is diagonalizable with real eigenvalues, provided that the number of filters $k \asymp_M \log^2(T)$, the regret is $\polylog(T)$ with high probability and the regret bound is independent of both transition matrix spectral radius $\rho(A)$ (related to mixing rate) and closed-loop matrix spectral radius $\rho(G)$ (related to forecast memory). 

\begin{remark}
\normalfont Note that for any matrix $A$, there exists a constant $\gamma\geq 1$ such that $\|A^t\|_2 \leq \gamma t^{\log (\gamma)}$ \cite{kozyakin2009accuracy}. We justify our assumption on diagonalizable $G$ with real eigenvalues in the following section. The filter quadratic condition is easily verified for $s > 2(k+1)$ and $t_0 \gtrsim_{R_\Theta, m, \gamma, \beta, \delta} k^2\log(k)$ for all $A$ with $\rho(A) \leq 1$ for the filters corresponding to truncated observations (a.k.a. basis vectors) such as in \citet{tsiamis2020online}. When $A$ is symmetric, this condition can be further simplified to $t \Omega_{s/2}(D; \psi) - \Omega_{t+1}(D; \psi) \succeq 0$ for all diagonal matrices $D$ with $|D_{ii}| \leq 1$.
\end{remark}

\section{Approximation error: Generalized Kolmogorov width}\label{sec:kolmogorov}

\subsection{Width of a subset}\label{sec:generalized_kwidth}
The SLIP algorithm is based on approximating the Kalman predictive model. In this section, we start by introducing a generalization of \textit{Kolmogorov $k$-width of a subset}, which is a criterion to assess the quality of a function approximation method. We then present our approximation technique which gives the SLIP algorithm.

\begin{definition} \label{def:k-width}\normalfont{\textbf{(Generalized Kolmogorov $\mathbf{k}$-width)}} \textit{Let $W$ be a subset in a normed linear space with norm $\|.\|$ whose elements are $d \times n$ matrices. Given $d \times n$ matrices $u_1, \dots, u_k$ for $k \geq 1$, let
\begin{align*}
    U(u_1, \dots, u_k) \triangleq \Big\{y \; \Big| \; y = \sum_{i=1}^k a_i u_i, \; \forall a_i \in \mathbb{R}^{d \times d}\Big\}
\end{align*}
be the subset constructed by linear combinations of $u_1, \dots, u_k$ with coefficient matrices $a_1, \dots, a_k$. For a fixed $k \geq 1$, denote by $\mathcal{U}_k$ the set of $U(u_1, \dots, u_k)$ for all possible choices of $u_1, \dots, u_k$:
\begin{align*}
    \mathcal{U}_k \triangleq \big\{ U(u_1, \dots, u_k) \; \big| \; \forall u_i \in \mathbb{R}^{d \times n} \big\}.
\end{align*}
The generalized $k$-width of $W$ is defined as 
\begin{align*}
    d_k(W) \triangleq \inf_{U \in \mathcal{U}_k} \sup_{x\in W} \text{dist}(x;U) = \inf_{U \in \mathcal{U}_k} \sup_{x\in W}\inf_{y\in U} \|x-y\|,
\end{align*}
where $\text{dist}(x;U)$ is the distance of $x$ to subset $U$ and the first infimum is taken over all subsets $U \in \mathcal{U}_k$.}
\end{definition}
Here, we are interested in approximating $W$ with the ``best'' subset in the set $\mathcal{U}_k$: the subset that would minimize the worst case projection error of $x \in W$ among all subsets in $\mathcal{U}_k$. This minimal error is given by the generalized $k$-width of $W$. Figure \ref{fig:kolmogorov_width} illustrates an example in which $W$ is an ellipsoid in $\mathbb{R}^3$ and we are interested in approximating it with a 2-dimensional plane $(k = 2)$. In this example, $\mathcal{U}_2$ is the set of all planes and plane $U$ offers the smallest worst-case projection error $d_2(W)$ for approximating $W$.

Definition \ref{def:k-width} generalizes the original Kolmogorov $k$-width definition in two ways. First, in our definition $W$ is allowed to be a subset of matrices whereas in the original Kolmgorov width, $W$ is a subset of vectors. This generalization is necessary as we wish to approximate the coefficient set of the Kalman predictive model whose elements $\cO_t$ and $\cC_t$ are matrices. Second, we allow the coefficients $a_i$ to be matrices, generalizing over the scalar coefficients used in the original definition of Kolmogorov width. When constructing a reparameterization, a linear predictive model yields a convex objective regardless of whether the coefficients are matrices or scalars. Allowing coefficients to be matrices as opposed to restricting them to be scalars gives flexibility to find a reparameterization with small approximation error, as demonstrated in Theorem \ref{thm:kwidthresult}.

\subsection{From a small width to an efficient convex relaxation}
Before stating our approximation technique, we briefly describe how a small generalized $k$-width allows for an efficient convex relaxation. The ideas presented in this section will be made more concrete in subsequent sections.

To understand the main idea, consider system \eqref{eq:lds_definition} with no inputs whose predictive model can be written as $m_{t+1} = \cO_t y_{1:t}$. Matrix $\cO_t$ belongs to a subset in $\mathbb{R}^{m \times mt}$ restricted by the constraints on system parameters. A naive approach for a convex relaxation is learning $\cO_t$ in the linear predictive model $\cO_t y_{1:t}$ directly. However in this approach, the total number of parameters is $m^2 t$, which hinders achieving sub-linear regret.

Now suppose that there exists $k \ll t$ for which the generalized $k$-width is small, i.e. there exist fixed known matrices $u_1, \dots, u_k \in \mathbb{R}^{m \times mt}$ that approximate any $\cO_t$ with a small error $\cO_t \approx \sum_{i=1}^k a_i u_i,$ where $a_1, \dots, a_k \in \mathbb{R}^{m \times m}$ are coefficient matrices. The predictive model can be approximated by
\begin{align*}
    m_{t+1} \approx \sum_{i=1}^k a_i u_i y_{1:t},
\end{align*}
provided that norm of $y_{1:t}$ (compared to the approximation error of $\cO_t$) is controlled with high probability. Since $u_i$ and $y_{1:t}$ are known, we only need to learn coefficients $a_1, \dots, a_k$ resulting in a total of $m^2k$ parameters which is much smaller than the naive approach with $m^2 t$ parameters.
\begin{figure}
  \begin{center}
    \includegraphics[width=0.45\textwidth]{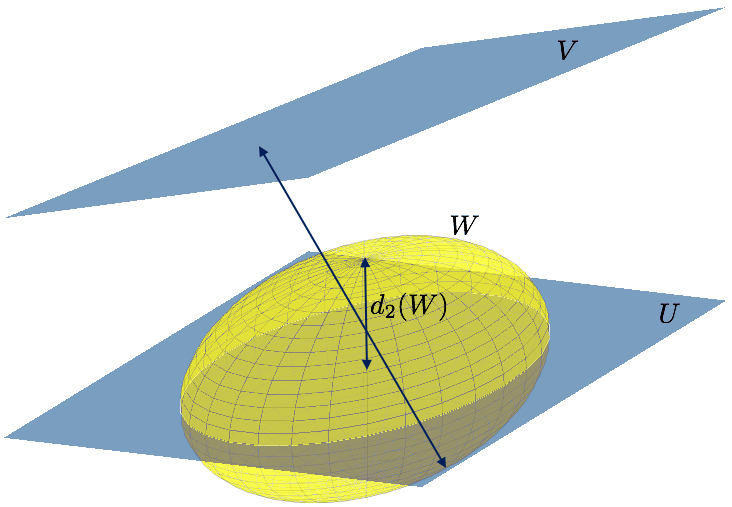}
    \caption{Approximating $W$, a 3D ellipsoid, by a 2D plane $U(u_1, u_2)$ among $\mathcal{U}_2$, the set of all planes. In this example, $U$ has the smallest worst-case projection error that is equal to the $2$-width of $W$ denoted by $d_2(W)$.}
   \label{fig:kolmogorov_width}
  \end{center}
\end{figure}

\subsection{Filter approximation}\label{sec:filter_approximation}
Consider the matrix 
\begin{align}\notag 
    \mu(G) \triangleq [I,G,G^2,\ldots,G^{T-1}],
\end{align}
where $G \in \mathbb{R}^{d\times d}$ is a real square matrix with spectral radius $\rho(G)\leq 1$. We seek to approximate $\mu(G) \approx \widetilde{\mu}(G) = \sum_{i=1}^k a_i u_i$ by a linear combination of $k$ matrices $u_1, \dots, u_k \in \mathbb{R}^{d \times Td}$ and coefficient matrices $\{a_1, \dots, a_k\} \in \mathbb{R}^{d \times d}$. We evaluate the quality of approximation in operator 2-norm $\|\mu(G) - \widetilde{\mu}(G)\|_2$ by studying the generalized $k$-width of $\mu(G)$.

We demonstrate a sharp phase transition. Precisely, we show that when $G$ is diagonalizable with real eigenvalues, the width $d_k(W)$ decays exponentially fast with $k$, but for a general $G$ with $d \geq 2$ it decays only polynomially fast. In other words, when $d \geq 2$ the inherent structure of the set $W$ is not easily exploited by linear subspaces.

\begin{theorem}\normalfont{\textbf{(Kalman filter $\mathbf{k}$-width)}}\label{thm:kwidthresult}
\textit{Let 
\begin{align}\notag 
     W \triangleq \Big\{\mu(G) = [I,G,G^2,\ldots,G^{T-1}] \; \Big| \; G\in \mathbb{R}^{d\times d}, \rho(G)\leq 1\Big\}
\end{align}
and endow the space of $W$ with the 2-norm. The following bounds hold on the generalized $k$-width of the set $W$.
\begin{enumerate}[1.]
    \item If $d\geq 2$, then for $1\leq k\leq T$,
    \begin{align}\notag 
        d_k(W) \geq \sqrt{T-k}. 
    \end{align}
    \item Restrict $G$ to be diagonalizable with real eigenvalues. If $T\geq 10$, then for any $d \geq 1$
    \begin{align}\notag 
        d_k(W) \leq C_0 d \sqrt{T} (\log T)^{1/4} c^{-k/\log T},
    \end{align}
    where $c = \exp(\pi^2/16)$ and $C_0 = \sqrt{43}$. Moreover, there exists an efficient spectral method to compute a $k$-dimensional subspace that satisfies this upper bound.
\end{enumerate}}
\end{theorem}
\begin{proof}
Here, we only provide a proof sketch; see Appendix \ref{app:kolmogorov_width} for a complete proof.

Let $\lambda_1, \dots \lambda_d \in [-1,1]$ be the eigenvalues of $G$. Let $v_i$ be the right eigenvectors of $G$ and $w^\top_i$ be the left eigenvectors of $G$ and write
\begin{align*}
    \mu(G) = \sum_{i=1}^d v_i w^\top_i ([1, \lambda_i, \dots, \lambda_i^{T-1}]\otimes I_d) = \sum_{i=1}^d v_i w^\top_i (\mu(\lambda_i) \otimes I_d).
\end{align*}

We approximate the row vector $\mu(\lambda)$ for any $\lambda \in [-1,1]$ using principal component analysis (PCA). The covariance matrix of $\mu(\lambda)$ with respect to a uniform measure is given by 
\begin{align}\notag 
    H = \int_{\lambda = -1}^1 \frac{1}{2} \mu(\lambda)^\top\mu(\lambda)  d\lambda \quad \Rightarrow \quad H_{ij} = \int_{-1}^1 \frac{1}{2} \lambda^{i-1} \lambda^{j-1} d \lambda = \frac{(-1)^{i+j}+1}{2(i+j-1)}.
\end{align}
Let $\{\phi_j\}_{j=1}^k$ be the top $k$ eigenvectors of $H$. We approximate $\mu(\lambda)$ by $\widetilde{\mu}(\lambda) = \sum_{j=1}^k \langle \mu^\top(\lambda), \phi_j \rangle \phi_j^\top$ and thus obtain
\begin{align*}
    \mu(G) \approx \widetilde{\mu}(G) = \sum_{j=1}^k \Big[\sum_{i=1}^d \langle \mu^\top(\lambda_i), \phi_j \rangle v_i w^\top_i \Big] (\phi^\top_j \otimes I_d) = \sum_{j=1}^k a_j u_j.
\end{align*}

We show a uniform bound on $\|\mu(G) - \widetilde{\mu}(G)\|$ by first analyzing the PCA approximation error which depends on the spectrum of matrix $H$. Matrix $H$ is a positive semi-definite Hankel matrix, a square matrix whose $ij$-th entry only depends on the sum $i+j$. We leverage a recent result by \Citet{beckermann2017singular} who proved that the spectrum of positive semi-definite Hankel matrices decays exponentially fast. 

This result, however, only guarantees a small \textit{average} error but we need to prove that the \textit{maximum} error is small to ensure a uniform bound on regret. Observe that the PCA error $r(\lambda) = \mu(\lambda) - \widetilde{\mu}(\lambda)$ is defined over a finite interval $[-1,1]$ with a small average. Thus, by computing the Lipschitz constant of $r(\lambda)$, we show that the maximum approximation error is small, resulting in an upper bound on $d_k(W)$.

For the first claim, we lower bound the generalized $k$-width of $W$ by relaxing the sup-norm by a weighted average, resulting in a \emph{weighted} version of generalized $k$-width. We observe that the weighted $k$-width can be computed using PCA. We compute the approximation error of PCA showing that this error is large.
\end{proof}

The approximation technique used in the above theorem can readily be applied to approximate the coefficients of the Kalman predictive model by
\begin{align*}
    \widetilde{\cO}_t & = \sum_{j=1}^k \Big[ \sum_{i=1}^d \langle \mu(\lambda_i)^\top, \phi_j \rangle C v_i w_i^\top K\Big] (\phi^\top_j(t:1) \otimes I_m),\\
    \widetilde{\mathcal{C}}_t  & = \sum_{j=1}^k \Big[ \sum_{i=1}^d \langle \mu(\lambda_i)^\top, \phi_j \rangle C v_i w_i^\top (B-KD)\Big] (\phi^\top_j(t:1) \otimes I_n),
\end{align*}
where we used the fact that $[\lambda_i^{t-1}, \dots, \lambda_i, 1]$ can be approximated by truncated eigenvectors $\{\phi_j(t:1)\}_{j=1}^k$. The relaxed model $\widetilde{m}_{t} \triangleq \widetilde{\cO}_ty_{1:t-1} + \widetilde{\mathcal{C}}_t x_{1:t-1} + Dx_{t}$ can be written in the form $\widetilde{m}_{t} = \widetilde{\Theta} f_{t}$. 
The feature vector $f_{t}$ is defined in \eqref{eq:feature_def} and the parameter matrix $\widetilde{\Theta}$ is obtained by concatenating the corresponding coefficient matrices as described below
\begin{align}\label{eq:pca_parameters}
    \widetilde{\Theta} = \vast[ \underbrace{\Big[ \sum_{i=1}^d \langle \mu(\lambda_i)^\top, \phi_j \rangle C v_i w_i^\top K \Big]_{j=1}^k}_{\substack{\in \mathbb{R}^{m \times mk} \\ \text{for output features}}} \; \vast| \; \underbrace{\Big[ \sum_{i=1}^d \langle \mu(\lambda_i)^\top, \phi_j \rangle C v_i w_i^\top (B-KD) \Big]_{j=1}^k}_{\substack{\in \mathbb{R}^{m \times nk} \\ \text{for input features}}} \; \vast| \; \underbrace{ \vphantom{\Big[ \sum_{i=1}^d \langle \mu(\lambda_i)^\top, \phi_j \rangle C v_i w_i^\top (B-KD) \Big]} D}_{\substack{\in \mathbb{R}^{m \times n} \\ \text{for } x_t}} \vast]_{m \times l}
\end{align}
A complete derivation of convex relaxation along with an approximation error analysis is provided in Appendix \ref{sec:convex_relaxation_bias}.

\section{Proof roadmap of Theorem \ref{thm:regret_bound_paper}}\label{sec:proof_roadmap}

In this section we present a proof sketch for Theorem \ref{thm:regret_bound_paper}; the complete proof is deferred to Appendix \ref{app:regret_decomposition} and Appendix \ref{app:regret_analysis}. Let $e_t = y_t - m_t$ denote the innovation process and $b_t = \widetilde{m}_t - m_t$ denote the bias due to convex relaxation. Define
\begin{align}\label{eq:risk_definition_L_t}
    \cL(T) \triangleq \sum_{t=1}^T \|\hat{m}_t - m_t\|_2^2.
\end{align}
$\cL(T)$ measures the difference between Algorithm \ref{alg:slip} predictions and the Kalman predictions in hindsight. Regret defined in \eqref{eq:regret_definition} can be written as 
\begin{align}\label{eq:regret_decomposition_to_squared_error_cross_term}
    \text{Regret}(T) = \sum_{t=1}^T \|\hat{m}_t - m_t\|_2^2 - \sum_{t=1}^T 2e_t^\top (\hat{m}_t - m_t) = \cL(T) - \sum_{t=1}^T 2e_t^\top (\hat{m}_t - m_t).
\end{align}
Using an argument based on self-normalizing martingales, the second term is shown to be of order $\sqrt{\cL(T)}$ and thus, it suffices to establish a bound on $\cL(T)$.
Define 
\begin{align}\label{eq:definitions_of_Z_E_B}
    Z_t \triangleq \alpha I + \sum_{i=1}^t f_if_i^\top, \qquad E_t \triangleq \sum_{i=1}^t e_i f_i^\top, \qquad  B_t \triangleq \sum_{i=1}^{t} b_i f_i^\top.
\end{align}
A straighforward decomposition of loss gives
\begin{align}\label{eq:regret_decomp_roadmap}
    \cL(T) \leq \underbrace{3 \sum_{i=1}^T \|E_{t-1} Z_{t-1}^{-1} f_t\|_2^2}_{\text{least squares error}} + \underbrace{3 \sum_{i=1}^T \|B_{t-1} Z_{t-1}^{-1} f_t+ b_t\|_2^2}_{\text{improper learning bias}} + \underbrace{3  \sum_{i=1}^T \|\alpha \tilde{\Theta} Z_{t-1}^{-1} f_t\|_2^2}_{\text{regularization error}}.
\end{align}
\subsection{Least squares error}
Among all, it is most difficult to establish a bound on the least squares error. Consider the following upper bound
\begin{align*}
    \sum_{t=1}^T \|E_{t-1} Z_{t-1}^{-1} f_t\|_2 \leq \max_{1 \leq t \leq T} \|E_{t-1} Z_{t-1}^{-1/2}\|_2 \sum_{t=1}^T \| Z_{t-1}^{-1/2} f_t\|_2.
\end{align*}
We show the first term is bounded by $\polylog(T)$ for any $\delta \geq 0$. In particular,
\begin{align*}
    \max_{1 \leq t \leq T} \|E_{t-1} Z_{t-1}^{-1/2}\|_2 & \lesssim_{R_\Theta, m, \gamma, \beta, \delta} \max_{1 \leq t \leq T} \log \Big( \frac{\det(Z_t)\det(\alpha I)^{-1}}{ \delta}\Big) \lesssim_{R_\Theta, m, \gamma, \beta, \delta} k \log(T). 
\end{align*}
Our argument is based on vector self-normalizing martingales, a similar technique used by \citet{abbasi2011improved, sarkar2018near, tsiamis2020online}. $\det(Z_t)$ is bounded by $\poly(T)$ for two reasons. First, the feature dimension, which is linear in the number of filters $k$, is $\polylog(T)$ on account of Theorem \ref{thm:kwidthresult}. Second, the marginal stability assumption ($\rho(A) \leq 1$) ensures that features and thus $Z_t$ grow at most polynomially in $t$. 

It remains to prove that the summation $\sum_{t=1}^T \|Z_{t-1}^{-1/2} f_t\|_2^2$ is bounded by $\polylog(T)$ with high probability. We use an argument inspired by Lemma 2 of \citet{lai1982least} and Schur complement lemma \cite{zhang2006schur} to conclude that 
\begin{align*}
    \sum_{t=1}^T \|Z_{t-1}^{-1/2} f_t\|_2^2 \asymp_M \polylog(T) \quad \Leftrightarrow \quad Z_t - \frac{1}{c_T} f_t f_t^\top \succeq 0 \quad \text{for} \quad c_T \asymp_M \polylog (T).
\end{align*}
Therefore, it suffices to prove the right-hand side. We show a high probability L\"owner upper bound on $f_t f_t^\top$ based on the feature covariance $\cov(f_t)$ using sub-Gaussian quadratic tail bounds \Citep{vershynin2018high}. To capture the excitation behavior of features, we establish a L\"owner lower bound on $Z_t$ by proving that the process $\{f_t\}_{t \geq 1}$ satisfies a \textit{martingale small-ball condition} \cite{mendelson2014learning, simchowitz2018learning}. We leverage the small-ball condition lower tail bounds and prove the following lemma.

\begin{lemma}\label{lemma:excitation_result} \normalfont{\textbf{(Martingale small-ball condition)}} \textit{Let $\phi_1, \dots, \phi_k \in \mathbb{R}^T$ be orthonormal and fix $\delta > 0$. Given system \eqref{eq:lds_definition}, let $\mathcal{F}_t = \sigma\{\eta_0,\dots, \eta_{t-1}, \zeta_1, \dots, \zeta_t\}$ be a filteration and for all $t \geq 1$ define
\begin{align*}
    f_t = \psi_1 \otimes y_{t-1} + \dots + \psi_{t-1} \otimes y_1, \quad \text{where} \quad \psi_i = [\phi_1(i), \dots, \phi_k(i)]^\top.
\end{align*}
Let $\Gamma_i = \cov(f_{t+i}|\mathcal{F}_t)$.
\begin{enumerate}[1.]
    \item For any $1 \leq s \leq T$, the process $\{f_t\}_{t \geq 1}$ satisfies a $(s, \Gamma_{s/2}, p = 3/20)$-block martingale small-ball (BMSB) condition, i.e. for any $t \geq 0$ and any fixed $\omega$ in unit sphere $\mathcal{S}^{l-1}$
    \begin{align*}
        \frac{1}{s}\sum_{i=1}^s \mathbb{P}\Big(|\omega^\top f_{t+i}|\geq \sqrt{\omega^\top \Gamma_{s/2} \omega} \mid \mathcal{F}_t \Big) \geq p.
    \end{align*}
    \item Under the assumptions of Theorem \ref{thm:regret_bound_paper}, the following holds with probability at least $1-\delta$
    \begin{align*}
        \sum_{t=1}^T \|Z_{t-1}^{-1/2}f_t\|_2^2 \leq \kappa k^2 \log(T) \poly(R_\Theta, \beta, m, \log(\gamma), \log\Big(\frac{1}{\delta}\Big)).
    \end{align*}
\end{enumerate}}
\end{lemma}
Provided that the number of filters is $\polylog(T)$, the above lemma ensures that $\sum_{t=1}^T \|Z_{t-1}^{-1/2}f_t\|_2^2$ is also $\polylog(T)$, which is the desired result.

\subsection{Improper learning bias}
We characterize the improper learning bias term in \eqref{eq:regret_decomp_roadmap} by first showing a uniform high probability bound on the convex relaxation error stated in the theorem below. The proof can be found in Appendix \ref{sec:convex_relaxation_bias}. 

\begin{theorem} \label{thm:convex_relaxation_informal} \normalfont{\textbf{(Convex relaxation error bound, informal)}} \textit{Consider system \eqref{eq:lds_definition} with bounded inputs $\|x_t\|_2 \leq R_x$ and assume conditions (i)-(ii) of Theorem \ref{thm:regret_bound_paper} holds. Then for any $\epsilon, \gamma \geq 0$, if the number of filters $k$ satisfies $ k  \gtrsim_M \log(T) \log(T/\epsilon)$,
then the following holds for $\widetilde{\Theta}$ as defined in \eqref{eq:pca_parameters}
\begin{align*}
    \Prob\Big[ \|\widetilde{\Theta} f_t - m_t\|_2^2 \geq \epsilon \Big] \leq \delta.
\end{align*}}
\end{theorem}

In Appendix \ref{app:proof_thm_1}, the result of the above theorem is followed by an application of a vector self-normalizing martingale theorem to prove a $\polylog(T)$ bound on the improper learning bias.

\begin{remark} \normalfont While the algorithm derivation, convex relaxation approximation error, and most of the regret analysis consider a system with control inputs, the excitation result of Lemma \ref{lemma:excitation_result} is given without inputs. We believe that extending our analysis for LDS with inputs is possible by characterizing input features and in light of the experiments. However, such an extension requires some care. For instance, one needs to characterize the covariance between features constructed from observations and features constructed from inputs to demonstrate a small-ball condition.
\end{remark}

\subsection{Regularization error}
Lastly, we demonstrate an upper bound on the regularization error in \eqref{eq:regret_decomp_roadmap}. We write the following bound
\begin{align*}
    \sum_{t=1}^T \|\alpha \tilde{\Theta} Z_{t-1}^{-1} f_t\|_2^2 \leq \alpha^2 \frac{1}{\alpha}\|\tilde{\Theta}\|_2^2 \sum_{t=1}^T \|Z_{t-1}^{-1/2} f_t\|_2^2 \leq \sum_{t=1}^T \|Z_{t-1}^{-1/2} f_t\|_2^2.
\end{align*}
The first inequality is based on $Z_t \succeq \alpha I$ and the submultiplicative property of norm. The second inequality uses the fact that $\|\widetilde{\Theta}\|_2^2 \leq 1/\alpha$ for $\alpha \asymp_M (R_\Theta k T^\beta)^{-1}$ as shown in Appendix \ref{app:regularization_term}. The last term is bounded as result of Lemma \ref{lemma:excitation_result}.

\section{Experiments}\label{sec:experiments}

We carry out experiments to evaluate the empirical performance of our provable method in three dynamical systems with long-term memory. We compare our results against those yielded by the wave filtering algorithm \cite{hazan2017learning} implemented with follow the regularized leader and the truncated filtering algorithm \cite{tsiamis2020online}. We consider $\|\hat{m}_t - m_t\|^2$, the squared error between algorithms predictions and predictions by a Kalman filtering algorithm that knows system parameters, as a performance measure. For all algorithms, we use $k = 20$ filters and run each experiment independently 100 times and present the average error with 99\% confidence intervals.
    
\begin{figure}[h]
    \centering
    \includegraphics[scale = 0.33]{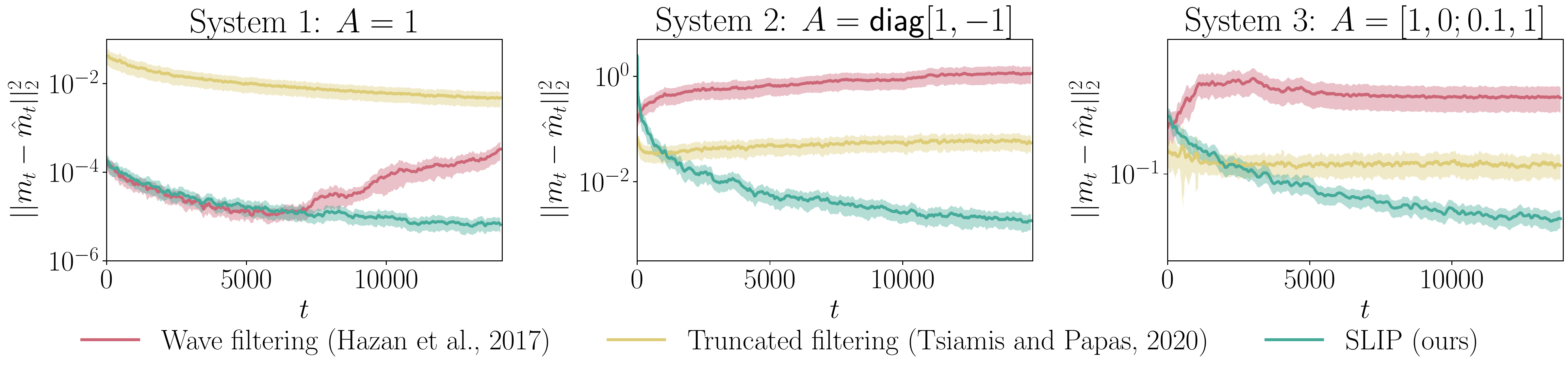}
    \caption{Performance of our algorithm compared with wave filtering and truncated filtering. System 1 is an scalar LDS with $A = B = D = 1$, $C = Q = R = 0.001$, and $x_t \sim \mathcal{N}(0,2)$. System 2 is a multi-dimensional LDS with no inputs and $A = \mathsf{diag}[-1,1]$, $C = [0.1,0.5]$, $R = 0.5$, and $Q = [4, 6; 6,10]\times 10^{-3}$. System 3 is another multi-dimensional LDS with non-symmetric $A = [1, 0; 0.1, 1]$, $x_i \sim \mathcal{U}(-0.01,0.01)$, $Q = 10^{-3}I$, $R = I$, $C = [0,0.1;0.1,1]$, and $B, D$ are matrices of all ones.}
    \label{fig:experiments}
\end{figure}

In the first example (Figure \ref{fig:experiments}, left), we consider a scalar marginally stable system with $A=1$ and Gaussian inputs. This system exhibits long forecast memory with $G \approx 0.999$. Observe that the truncated filter suffers from a large error which is due to ignoring long-term dependencies. The wave filter predictions also deviates from optimal predictions as it only considers $y_{t-1}, x_{1:t}$ for predicting $y_t$. The middle plot in Figure \ref{fig:experiments} presents the results for a multi-dimensional system with $A = \mathsf{diag}[-1,1]$ and no inputs. This system also has a long forecast memory ($G$ has eigenvalues $\approx \{0.991, -0.932\}$), resulting in poor performance of the truncated filter. The wave filter also performs poorly in this system as it is only driven by stochastic noise. For the last example, we consider another multi-dimensional system where $A$ is a lower triangular matrix (Figure \ref{fig:experiments}, right). This is a difficult example where $\rho(A) = 1$ but $\|A\|_2 > 1$, resulting in a polynomial growth of the observations over time. The results show that our algorithm outperforms both the wave filter, which requires a symmetric $A$, and the truncated filter in the case of fast-growing observations.

Experiments on hyperparameter sensitivity of our algorithm and comparison with the EM algorithm are provided in Appendix \ref{app:extra_experiments}.

\section{Discussion and future work}
We presented the SLIP algorithm, an efficient algorithm for learning a predictive model of an unknown LDS. Our algorithm provably and empirically converges to the optimal predictions of the Kalman filter given the true system parameters, even in the presence of long forecast memory. We analyzed the generalized $k$-width of the Kalman filter coefficient set with closed-loop matrix $G$ and obtained a low-dimensional linear approximation of the Kalman filter when $G$ is diagonalizable with real eigenvalues. We proved that without assuming real eigenvalues, the Kalman filter coefficient set is difficult to approximate by linear subspaces. Our approach of studying $k$-width as a measure for the possibility of an efficient convex relaxation may be of independent interest. Important future directions are to design efficient algorithms that handle arbitrary $G$ and to provide theoretically guaranteed uncertainty estimation for prediction.

\section*{Acknowledgements}
The authors would like to thank the anonymous reviewers for their comments and suggestions, which helped improve the quality and clarity of the manuscript.  This work is supported by the Scalable Collaborative Human-Robot Learning (SCHooL) Project, an NSF National Robotics Initiative Award 1734633. The work of Jiantao Jiao was partially supported by NSF Grants IIS-1901252 and CCF-1909499.

\bibliographystyle{plainnat}
\bibliography{references}
\newpage

\appendix
\section*{Guide to the appendix}
\addcontentsline{toc}{section}{\protect Guide to the Appendix}

The appendix is organized as follows.

In Appendix \ref{app:matrix_representations}, we present a matrix representation of system \eqref{eq:lds_definition} describing aggregated observations $y_{1:t}$ in terms of past inputs and noise. We also restate our matrix representation of the Kalman predictive model.

In Appendix \ref{app:norm_bounds}, we provide upper bounds on the matrix coefficients used in the aggregated system representation as well as a high probability upper bound on the norm of observations $\|y_{1:t}\|_2$. We also discuss our assumption on the 2-norm of the Kalman coefficient matrices (control matrix $\cC_t$ and observation matrix $\cO_t$) and present two examples providing bounds on the 2-norm of these coefficients.

In Appendix \ref{app:kolmogorov_width}, we first analyze the error of approximating $\mu(\lambda)$ by spectral methods, considering the spectrum of the Hankel covariance matrix. A proof of Theorem \ref{thm:kwidthresult} is presented in Appendix \ref{app:proof_kwidth_theorem}.

In Appendix \ref{sec:convex_relaxation_bias}, we analyze convex relaxation approximation error and show that the convex relaxation bias is small with high probability, provided that the number of filters $k \gtrsim_M \log^2 (T)$.

In Appendix \ref{app:regret_decomposition}, we write a bound on regret decomposed into least squares error, improper learning bias, regularization error, and innovation error. We further extract the term $\|Z_{t-1}^{-1/2} f_t\|_2^2$ making the bound ready for analysis in subsequent sections.

In Appendix \ref{app:regret_analysis}, we provide our regret analysis. In Appendix \ref{app:det_z_t_bound}, we present a high probability bound on $\det(Z_t)$ that appears multiple times throughout our analysis. In Appendix \ref{app:self_normalizing_martingales}, we derive a result on self-normalizing vector martingales that assists bounding several terms. In Appendix \ref{app:bound_Zt_inv_ft}, we provide a bound on $\|Z_{t-1}^{-1/2} f_t\|_2^2$ using sub-Gaussian tail properties, a block-martingale small-ball condition, and a filter quadratic function condition. The proof of Lemma \ref{lemma:excitation_result} is given in Appendix \ref{app:proof_lemma_1}. The regularization term and innovation error are analyzed in Appendix \ref{app:regularization_term} and Appendix \ref{app:innovation_error}, respectively. The proof of the regret theorem is presented in Appendix \ref{app:proof_thm_1}.

A few technical lemmas are presented in Appendix \ref{app:technical_lemmas}. Additional experiments are presented in Appendix \ref{app:extra_experiments}.
\newpage 

\section{Aggregated representations}\label{app:matrix_representations}
We start by introducing an aggregated notation for representing linear dynamical systems and the Kalman predictive model. 

\subsection{Linear dynamical systems}\label{app:matrix_rep_LDS}
For the linear dynamical system of \eqref{eq:lds_definition}, define the following matrices
\begin{align}\label{eq:lds_definition_of_matrices}
\begin{split}
    \mathcal{T}_t & = \begin{bmatrix}
    C & 0 & 0 & \dots & 0\\
    CA & C & 0 & \dots & 0\\
    CA^2 & CA & C & \dots & 0\\
    \vdots & \vdots & \vdots & \ddots & \vdots\\
    CA^{t-1} & CA^{t-2} & CA^{t-3} & \dots & C
    \end{bmatrix} \begin{bmatrix}
    AP^{1/2} & 0 & 0 & \dots & 0\\
    0 & Q^{1/2} & 0 & \dots & 0\\
    0 & 0 & Q^{1/2} & \dots & 0\\
    \vdots & \vdots & \vdots & \ddots & \vdots\\
    0 & 0 & 0 & \dots & Q^{1/2}
    \end{bmatrix},\\ 
    \mathcal{I}_t & = \begin{bmatrix}
    D & 0 & 0 & \dots & 0\\
    CB & D & 0 & \dots & 0\\
    CAB & CB & D & \dots & 0\\
    \vdots & \vdots & \vdots & \ddots & \vdots\\
    CA^{t-2}B & CA^{t-3}B & CA^{t-4}B & \dots & D
    \end{bmatrix},\\
    \mathcal{R}_t & = \begin{bmatrix}
    R^{1/2} & 0 & 0 & \dots & 0\\
    0 &  R^{1/2} & 0 & \dots & 0\\
    0 & 0 &  R^{1/2} & \dots & 0\\
    \vdots & \vdots & \vdots & \ddots & \vdots\\
    0 & 0 & 0 & \dots &  R^{1/2}
    \end{bmatrix}.
\end{split}
\end{align}
Let $\mathcal{K}_t \mathcal{K}^\top_t = \mathcal{T}_t \mathcal{T}^\top _t+\mathcal{R}_t \mathcal{R}^\top _t$, where $\mathcal{K}_t$ is the unique solution to Cholesky decomposition. The system observations $y_{1:t}$ can be written as 
\begin{align}\label{eq:y_matrix_representation}
    y_{1:t} = \mathcal{K}_t \xi_{1:t} + \mathcal{I}_t x_{1:t},
\end{align}
where $\xi_i \in \mathbb{R}^m$ is a Gaussian random vector with covariance $I_m$.

\subsection{Kalman filter}\label{app:matrix_rep_KF}
For convenience, we restate our notation of the Kalman predictive model from Section \ref{sec:problem_statement}. Define the following matrices
\begin{align}
    \begin{split}
        \cO_t & = \begin{bmatrix} CG^{t-1}K & CG^{t-3}K & \dots & CK\end{bmatrix}, \\
        \mathcal{C}_t & = \begin{bmatrix}CG^{t-1}(B-KD) & CG^{t-2}(B-KD) & \dots & C(B-KD)\end{bmatrix}.
    \end{split}
\end{align}
We refer to $\cO_t$ and $\cC_t$ as \textit{observation matrix} and \textit{control matrix}, respectively. Using the above notation, the Kalman prediction $m_{t+1}$ is given by
\begin{align*}
    m_{t+1} = \cO_{t} y_{1:t} + \mathcal{C}_t x_{1:t} + Dx_{t+1}.
\end{align*}

\section{Norm bounds}\label{app:norm_bounds}
As a preliminary step, we compute a few bounds that will be used later in the regret analysis of the SLIP algorithm. In particular, we compute upper bounds on the norms of parameter matrices defined in \eqref{eq:lds_definition_of_matrices} and discuss upper bounds on the norms of observation and control matrix of the Kalman predictive model. Further, we derive a high probability upper bound on the observation norm. 

\subsection{Bounds on parameters}
The following lemma provides upper bounds on the norm of matrices that describe a linear dynamic system.
\begin{lemm} \label{lemma:lds_matrix_bounds}\normalfont{\textbf{(LDS parameter bounds)}} \textit{Consider system (4). Let $R_P = \max\{\|B\|_2, \|C\|_2, \|D\|_2 \}$ and $R_C = \max \{\|P\|_2, \|Q\|_2, \|R\|_2\}$. Suppose that $\|A^t\|_2 \leq \gamma t^{\log (\gamma)}$ for a bounded constant $\gamma \geq 1$. For $\mathcal{T}_t$, $\mathcal{I}_t$, and $\mathcal{K}_t$ defined in \eqref{eq:lds_definition_of_matrices}, the following operator norm bounds hold:}
\begin{enumerate}[(i)]
    \item $\|\mathcal{T}_t \|_2 \leq R_C^{1/2} R_P \gamma(1+\gamma) t^{\log(\gamma)+1}$,
    \item $\|\mathcal{I}_t\|_2 \leq R_P [1+ t \gamma t^{\log (\gamma)}]$,
    \item $\|\mathcal{K}_t\|_2\leq \sqrt{R_C + R_C R_P^2 (1+\gamma)^4 t^{2\log(\gamma)+2}}$.
\end{enumerate}
\end{lemm}
\begin{proof}
By Lemma \ref{lemma:triangular_block_toeplitz_2_norm},
\begin{align*}
    \|\mathcal{T}_t \|_2 \leq (\|A\|_2+1) R_C^{1/2} \|C\|_2 \sum_{i=1}^t \|A^i\|_2 \leq  R_C^{1/2} R_P \gamma(1+\gamma) t^{\log(\gamma)+1}.
\end{align*}
Similarly,
\begin{align*}
    \|\mathcal{I}_t\|_2 \leq \|D\|_2 + \|C\|_2\|B\|_2 \sum_{i=1}^t \|A^i\|_2 \leq R_P + R_P^2 \gamma t^{\log (\gamma)+1}.
\end{align*}
It follows by the sub-additive property of matrix operator norm that
\begin{align*}
    \|\mathcal{K}_t\mathcal{K}_t^\top\|_2 = \|\mathcal{K}_t\|^2_2 \leq \|\mathcal{T}_t \|_2^2 + \|\mathcal{R}_t \|_2^2 \quad \Rightarrow \quad \|\mathcal{K}_t\|_2\leq \sqrt{R_C + R_C R_P^2 (1+\gamma)^4 t^{2\log(\gamma)+2}}.
\end{align*}
\end{proof}

In the regret analysis, we assume that $\|\cO_t\|_2 \leq R_{\cO}t^\beta$ for a finite $\beta \geq 0$. We justify this assumption in the examples below. The following example shows that $\beta = 0$ when the system is single-input single-output (SISO).

\begin{examp} \normalfont{\textbf{(Observation matrix norm bound in SISO systems)}} \label{example:scalar_R_O_bound} \textit{For a SISO linear dynamical system, the following equation holds
\begin{align*}
    KC = \frac{A \Sigma^+ C^2 }{\Sigma^+ C^2 + R} \Rightarrow 0 \leq KC \leq A.
\end{align*}
We have $G = A - KC$. Applying the above constraint gives
\begin{align*}
    G \leq A
\end{align*}
The squared norm of vector $\cO_t$ is given by
\begin{align*}
    \|\cO_t\|_2^2 = \sum_{i=0}^{t-1} (KC G^i)^2 = \sum_{i=0}^{t-1} (A-G)^2 G^{2i}.
\end{align*}
Under the constraint $G \leq A \leq 1$, the maximum of $\|\cO_t\|_2^2$ is 1 obtained when $G = 0$ and $A = 1.$}
\end{examp}
In the following example, we compute a loose upper bound on $\|\cO_t\|_2$.
\begin{examp} \normalfont{\textbf{(Loose observation matrix norm bound in MIMO systems)}}
\textit{We begin by computing an upper bound on the norm of the Kalman gain. Let $K = AK'$. By the recursive updates of a stationary Kalman gain, we write
\begin{align*}
    CK' = C\Sigma^+ C^\top [C\Sigma^+ C^\top+Q]^{-1} \preceq I \Rightarrow \|CK'\|_2 \leq 1.
\end{align*}
Lower bounding $\|CK\|_2$ yields
\begin{align*}
    \|K'\|_2 \sigma_{\min}(C) \leq \|CK'\|_2 \leq 1 \Rightarrow \|K'\|_2 \leq \frac{1}{\sigma_{\min}(C)}.
\end{align*}
Let $\kappa_C = \sigma_{\max}(C)/\sigma_{\min}(C)$ to be the condition number of $C$. Assume $\|G^t\|_2 \leq \gamma_g t^{\log(\gamma_g)}$. We have
\begin{align*}
    \|\cO_t\|_2 \leq \sum_{i=1}^t \|C\|_2 \|G^i\|_2 \|K'\|_2 \leq \kappa_C \gamma_g t^{\log (\gamma_g)+1}.
\end{align*}}
\end{examp}

\subsection{Bound on observation norm}
One of the quantities that appear in the regret analysis of our algorithm is the squared norm of $y_{1:t}$. The following lemma provides a high probability upper bound for $\|y_{1:t}\|_2^2$.
\begin{lemm}\normalfont{\textbf{(Observation norm bound)}} \label{lemma:observation_norm_high_prob_bound}
\textit{Consider system (4). Let $R_P = \max\{\|B\|_2, \|C\|_2, \|D\|_2 \}$, $R_C = \max \{\|P\|_2, \|Q\|_2, \|R\|_2\}$, and $\|x_t\|_2 \leq R_x$. Suppose that $\|A^t\|_2 \leq \gamma t^{\log (\gamma)}$ for a bounded constant $\gamma \geq 1$. For any $\delta > 0$ and all $t \geq 0$,}
\begin{align*}
    \Prob \big[\|y_{1:t}\|_2^2 \geq 6 (R_P^2+1)(R_x^2+R_C) (1+\gamma)^4 (mt+\delta) t^{2+2\log (\gamma)} \big] \leq e^{-\delta}.
\end{align*}
\end{lemm}
\begin{proof} 
From \eqref{eq:y_matrix_representation}, we see that 
\begin{align*}
    \|y_{1:t}\|_2^2 \leq 2\|\mathcal{I}_t\|_2^2 \|x_{1:t}\|_2^2 + 2 \|\mathcal{K}_t\|_2^2\|\xi_{1:t}\|_2^2
\end{align*}
Using Gaussian upper tail bounds \cite{hsu2012tail}, we have
\begin{align*}
    \Prob \big[ \|\xi_{1:t}\|_2^2 > 2mt + 3\delta \big] \leq \Prob \big[ \|\xi_{1:t}\|_2^2 > mt + 2 \sqrt{mt \delta} + 2 \delta \big] \leq e^{-\delta}.
\end{align*}
Using the bounds computed in Lemma \ref{lemma:lds_matrix_bounds}, the following holds with probability at least $1-e^{-\delta}$
\begin{align}\label{ineq:bound_inequality_observation_norm}
    \begin{split}
        \|y_{1:t}\|_2^2
        \leq & 2\|\mathcal{I}_t\|_2^2 \|x_{1:t}\|_2^2 + 2 \|\mathcal{K}_t\|_2^2\|\xi_{1:t}\|_2^2\\
        \leq & 6 (R_P^2+1)(R_x^2+R_C) (1+\gamma)^4 (mt+\delta) t^{2+2\log (\gamma)}.
    \end{split}
\end{align}
\end{proof}

\section{Filter approximation and width analysis}\label{app:kolmogorov_width}
In this section we first provide a series of lemmas characterizing the reconstruction error of applying PCA to approximate the vector function $\mu(\lambda) = [1, \lambda, \dots, \lambda^{T-1}]$. These lemmas are later used to prove Theorem \ref{thm:kwidthresult}.

\subsection{Bounds on PCA approximation error}\label{app.pcaerror}
The goal of this section is to establish a uniform bound on the norm of the reconstruction error of approximating $\mu(\lambda)$ with $\widetilde{\mu}(\lambda)$. The following lemma states a standard result on the average PCA reconstruction error, presented here for completeness.
\begin{lemm} \normalfont{\textbf{(Average reconstruction error bound)}} \label{lemma:average_pca_error} \textit{Let $\mu(\lambda) \in \mathbb{R}^T$ be a vector function parameterized by $\lambda \in \mathcal{A}$. Define the following matrix with respect to probability measure $p$
\begin{align*}
    Z = \int_{\mathcal{A}} \mu(\lambda) \mu^\top(\lambda) p(d\lambda).
\end{align*}
Let $\{(\sigma_j, \phi_j)\}_{j=1}^T$ be the eigenpairs of $Z$. Let $\widetilde{\mu}(\lambda)$ be the projection of $\mu(\lambda)$ to the linear subspace spanned by $\{\phi_1, \dots, \phi_k\}$. Then,}
\begin{align*}
    \int_{\mathcal{A}} \|\mu(\lambda) - \widetilde{\mu}(\lambda)\|_2^2 p(d\lambda) = \sum_{j=k+1}^{T} \sigma_j.
\end{align*}
\end{lemm}
\begin{proof}
Define $U_k$ to be a $T \times k$ matrix with columns $\phi_1, \dots \phi_k$, the eigenvectors of matrix $Z$. The reconstruction error can be written as
\begin{align*} 
    r(\lambda) = \mu(\lambda) - U_kU_k^\top \mu(\lambda) = (I-U_kU_k^\top )\mu(\lambda) = \Pi_k \mu(\lambda). 
\end{align*}
The average squared norm of reconstruction error is given by
\begin{align*}
    \int_{\mathcal{A}} \|r(\lambda)\|_2^2 p(d\lambda) & = \int_{\mathcal{A}} \text{tr}[r(\lambda)r(\lambda)^\top] p(d\lambda) = \int_{\mathcal{A}} \text{tr}[\Pi_k \mu(\lambda)\mu(\lambda)^\top  \Pi_k^\top ] p(d\lambda)\\
    & =  \text{tr}[\Pi_k \int_{\mathcal{A}} \mu(\lambda)\mu(\lambda)^\top  p(d\lambda) \Pi_k^\top ] = \text{tr}[\Pi_k Z \Pi_k^\top ] = \sum_{j = k+1}^T \sigma_j.
\end{align*}
\end{proof}
We then use Lipschitz continuity of $\mu(\lambda)$ over the interval $[-1,1]$ to establish a uniform bound on the reconstruction error.
\begin{lemm}\label{lemma:max_reconstruction_bound}
\textit{Let $\mu(\lambda) = [1, \lambda, \lambda^2, \dots, \lambda^{T-1}]^\top$ for $\lambda \in [-1,1]$ and define
\begin{align*}
    H = \int_{-1}^1 \frac{1}{2} \mu(\lambda) \mu(\lambda)^\top d\lambda.
\end{align*}
Let $\{(\sigma_j, \phi^j)\}_{j=1}^{T}$ be the eigenpairs of $H$, where $\sigma_j$ are in decreasing order. Let $\widetilde{\mu}(\lambda)$ be the projection of $\mu(\lambda)$ to the linear subspace spanned by $\{\phi_1, \dots, \phi_k\}$. Then, for any $\lambda \in [-1,1]$ and $T \geq 1$,}
\begin{align*}
    \|\mu(\lambda) - \widetilde{\mu}(\lambda)\|_2^2 \leq T \sqrt{2\sum_{j = k+1}^{T} \sigma_j}.
\end{align*}
\end{lemm}
\begin{proof} Let us first compute an upper bound on the Lipschitz constant of $\mu(\lambda)$ over $\lambda \in [-1,1]$. The Lipschitz constant of $\mu(\lambda)$ is bounded by the norm of Jacobian $J(\mu(\lambda)) = [0, 1,2\lambda,\dots,(T-1)\lambda^{T-2}]$. Thus,
\begin{align*}
\frac{\|\mu(\lambda_2)-\mu(\lambda_1)\|_2}{|\lambda_2 - \lambda_1|} \leq \|J(\mu(\lambda))\|_2 \leq \sqrt{\sum_{t=1}^{T-1} t^2} \leq \sqrt{T^3/3}.
\end{align*}
Define $U_k$ to be a matrix with columns $\phi_1, \dots \phi_k$. The reconstruction error can be written as $r(\lambda) = (I-U_kU_k^\top )\mu(\lambda) = \Pi_k \mu(\lambda)$. A Lipschitz constant for reconstruction error norm is given by
\begin{align*}
    \|r(\lambda_2)\|_2 - \|r(\lambda_1)\|_2 & \leq \|r(\lambda_2) - r(\lambda_1)\|_2 \quad && \text{(inverse triangle inequality)}\\
    & = \|\Pi_k (\mu(\lambda_2) - \mu(\lambda_1))\|_2\\
    & \leq \|\Pi_k\|_2 \|(\mu(\lambda_2) - \mu(\lambda_1))\|_2 \quad && \text{(multiplicative property of norm)}\\
    & \leq \|(\mu(\lambda_2) - \mu(\lambda_1))\|_2 \quad && \text{($\Pi_k$ is contractive)}\\
    & \leq \sqrt{T^3/3} |\lambda_2 - \lambda_1| \quad && \text{(Lipschitz continuity of $\mu(\lambda)$)}
\end{align*}
Thus, an upper bound on the Lipschitz constant of $\|r(\lambda)\|_2^2$ can be computed
\begin{align*}
    \|r(\lambda_2)\|^2_2 - \|r(\lambda_1)\|^2_2 & = (\|r(\lambda_2)\|_2 - \|r(\lambda_1)\|_2)(\|r(\lambda_2)\|_2 + \|r(\lambda_1)\|_2) \\
    & \leq \big( \sqrt{T^3/3} |\lambda_2 - \lambda_1|\big) \big( 2 \max_\lambda \|r(\lambda)\|_2 \big)\\
    & \leq 2 \sqrt{T^3/3} \|\Pi_k\|_2 \max_{\lambda} \|\mu(\lambda)\|_2 |\lambda_2 - \lambda_1|\\
    & \leq 2 T^2 |\lambda_2 - \lambda_1|.
\end{align*}
Let $R_r = \max\limits_\lambda \|r(\lambda)\|_2^2$. On the account of Lemma \ref{lemma:average_pca_error}, $\|r(\lambda)\|_2^2$ has a bounded average over the interval $[-1,1]$. A bounded and $(2T^2)$-Lipschitz function that achieves the maximum $R_r$ has a triangular shape. It follows that
\begin{align*}
    \frac{R_r ^2}{2T^2} \geq \sum_{j = k+1}^{T} \sigma_j \quad \Rightarrow \quad  \|r(\lambda)\|_2^2 \leq R_r \leq T \sqrt{2\sum_{j = k+1}^{T} \sigma_j}.
\end{align*}
\end{proof}
In the following lemma, we prove that the PCA reconstruction error is small due to the exponential decay of the spectrum of the Hankel covariance matrix $H$.
\begin{lemm} \normalfont{\textbf{(Uniform bound on reconstruction error)}} \label{lemma:bounding_mu_approximation_error} \textit{Under the assumptions of Lemma \ref{lemma:max_reconstruction_bound} and for any $T \geq 10$}
\begin{align*}
    \|\mu(\lambda) - \widetilde{\mu}(\lambda)\|_2^2 \leq C_0 T \sqrt{\log T} c^{-k/\log T},
\end{align*}
\textit{where $c = \exp(\pi^2/8)$ and $C_0 = 43$. }
\end{lemm}

\begin{proof} We appeal to the following, which appears as Corollary 5.4 in \citet{beckermann2017singular}.
\begin{lemm}\label{lemma:hankel_matrix_singular_value_decay}
\textit{Let $H_n \in \mathbb{R}^{n \times n}$ be a positive semi-definite Hankel matrix. Then, }
\begin{align}\label{eq:hankel_singular_value_upperbound}
    \sigma_{j+2k} \leq 16 \Big[\exp \big( \frac{\pi^2 }{4 \log (8 \lfloor n/2 \rfloor/\pi)}\big) \Big]^{-2k+2} \sigma_j(H_n), \quad \text{for} \quad 1 \leq j+2k \leq n.
\end{align}
\end{lemm}
Setting $j = 1$ in \eqref{eq:hankel_singular_value_upperbound} with the assumption $T \geq 10$ yields
\begin{align*}
    \sigma_{2+2k} \leq \sigma_{1+2k} \leq 16 \sigma_1 \exp \Big(\frac{\pi^2}{4 \log T} \Big)^{-2k+2} \leq 1168 \sigma_1 \exp \Big(\frac{\pi^2}{4 \log T} \Big)^{-2k}.
\end{align*}
Let $c = \exp(\pi^2/8)$. It follows that
\begin{align*}
    \sigma_j \leq 1168 \sigma_1 c^{\frac{-2(j-2)}{\log T}} \leq 10512 \sigma_1 c^{\frac{-2j}{\log T}}.
\end{align*}
The largest singular value of Hankel matrix $H$ is bounded by
\begin{align*}
   \sigma_1 \leq \text{tr}(H) \leq \sum_{k=1}^{T} \frac{1}{2k+1} \leq \sum_{k=1}^T \frac{1}{k} - 1 \leq \log T,
\end{align*}
where the last inequality is due to a classic bound on the $T$-th harmonic number. We conclude from Lemma \ref{lemma:max_reconstruction_bound} that 
\begin{align*}
    \|\mu(\lambda) - \widetilde{\mu}(\lambda)\|_2^2 & \leq T \sqrt{21024  \sigma_1 \sum_{j=k+1}^T c^{-2j/\log T}} \\
    & \leq T \sqrt{21024 \log T \frac{c^{-2k/\log T}}{c^2-1}} \leq 43 T \sqrt{\log T} c^{-k/\log T}.
\end{align*}
\end{proof}

\subsection{Generalized Kolmogorov width analysis: Proof of Theorem \ref{thm:kwidthresult}}\label{app:proof_kwidth_theorem}
\begin{proofkwidth}
We first prove the second claim. Let $\lambda_1, \dots \lambda_d \in [-1,1]$ denote the eigenvalues of $G$. Let $v_i$ be the right eigenvectors of $G$ and $w^\top_i$ be the left eigenvectors of $G$. Eigendecomposition of $G^t$ implies $G^t = \sum_{i=1}^d v_i w_i^\top \lambda_i$. Therefore, matrix $\mu(G) = [I, G, \dots, G^{T-1}]$ can be written as
\begin{align*}
    \mu(G) = \sum_{i=1}^d v_i w^\top_i ([1, \lambda_i, \dots, \lambda_i^{T-1}]\otimes I_d) = \sum_{i=1}^d v_i w^\top_i (\mu(\lambda_i) \otimes I_d),
\end{align*}
where $\mu(\lambda_i) = [1, \lambda_i, \dots, \lambda_i^{T-1}]$ is a row vector. We approximate $\mu(\lambda)$ for any $\lambda \in [-1,1]$ using principal component analysis (PCA). The covariance matrix of $\mu(\lambda)$ with respect to a uniform measure is given by 
\begin{align}\notag 
    H = \int_{\lambda = -1}^1 \frac{1}{2} \mu(\lambda)^\top\mu(\lambda)  d\lambda \quad \Rightarrow \quad H_{ij} = \int_{-1}^1 \frac{1}{2} \lambda^{i-1} \lambda^{j-1} d \lambda = \frac{(-1)^{i+j}+1}{2(i+j-1)}.
\end{align}
Let $\{\phi_j\}_{j=1}^k$ be the top $k$ eigenvectors of $H$. We approximate $\mu(\lambda)$ by $\widetilde{\mu}(\lambda) = \sum_{j=1}^k \langle \mu^\top(\lambda), \phi_j \rangle \phi_j^\top$: 
\begin{align*}
    \mu(G) \approx \widetilde{\mu}(G) & = \sum_{i=1}^d v_i w_i^\top (\sum_{j=1}^k \langle \mu^\top(\lambda), \phi_j \rangle \phi_j^\top \otimes I_d)\\
    & = \sum_{j=1}^k \Big[\sum_{i=1}^d \langle \mu^\top(\lambda_i), \phi_j \rangle v_i w^\top_i \Big] (\phi^\top_j \otimes I_d) = \sum_{j=1}^k a_j u_j.
\end{align*}
Check that $a_1, \dots, a_k \in \mathbb{R}^{d \times d}$ and $u_1, \dots, u_k \in \mathbb{R}^{d \times dT}$. We have
\begin{align*}
    d_k(W) = \|\mu(G) - \widetilde{\mu}(G)\|_2 & = \|\sum_{i=1}^d v_i w_i^\top (\mu(\lambda_i) - \widetilde{\mu}(\lambda_i)) \otimes I_d\|_2 \\
    & \leq \sum_{i=1}^d \|\mu(\lambda_i) - \widetilde{\mu}(\lambda_i)\|_2 \\
    & \leq d \sup_{\lambda} \|\mu(\lambda) - \widetilde{\mu}(\lambda)\|_2.
\end{align*}
The first inequality uses subadditive and submultiplicative properties of norm and that $\|v_iw_i^\top\|_2 \leq 1, \|I_d\|_2 = 1$. By Lemma \ref{lemma:bounding_mu_approximation_error},
\begin{align*}
    d_k(W) \leq d \sqrt{43T} (\log T)^{1/4} \Big({\exp(\pi^2/16)}\Big)^{-k/\log T}.
\end{align*}

Now we prove the first claim by showing that the lower bound is realized for a particular set $W$. Since the case of $d = 2$ can be embedded as a subset for general $d\geq 2$ as the left top block, it suffices to show it for $d = 2$. We further constrain the set $W$ and only consider those $G$ with representation
\begin{align}\notag 
    G = \begin{bmatrix}a & b\\ -b& a \end{bmatrix},
\end{align}
where $a,b\in \mathbb{R}$. The eigenvalues of this matrix are complex numbers $a-jb$ and $a+jb$, which satisfy $\rho(G)\leq 1$ if $a^2 + b^2 \leq 1$, where $\rho(G)$ is the spectral radius of $G$. The nice property of this type of matrices is that there exists an explicit expression of $G^i$ for integer $i\geq 2$. Define complex number $z = a+jb$, then for integer $i\geq 0$: 
\begin{align}\notag 
    G^i = \begin{bmatrix} \Re(z^i) & \Im(z^i) \\ -\Im(z^i) & \Re(z^i)  \end{bmatrix},
\end{align}
where $\Re(z)$ represents the real part of complex number $z$, and $\Im(z)$ represents the imaginary part of $z$. 

We want to approximate $\mu(G) \in \mathbb{R}^{2 \times 2T}$ by $\sum_{i=1}^k a_i u_i$, where $a_i \in \mathbb{R}^{2 \times 2}$ and $u_i \in \mathbb{R}^{2 \times 2T}$. Let $W_1$ be the subset of row vectors realized by the first row of $\mu(G)$ for all $G \in \mathbb{R}^{2 \times 2}$ with $\rho(G) \leq 1$. We use the following property: the 2-norm of a matrix is lower bounded by the 2-norm of one of its rows. Based on this property, the 2-norm of error in approximating $\mu(G)$ is lower bounded by the 2-norm of error in approximating only one row of $\mu(G)$. Therefore, the generalized $k$-width of approximating $\mu(G)$ in 2-norm is lower bounded by the error of approximating the first row of $\mu(G)$ by a linear combination of $2k$ row vectors with dimension $2T$. In other words,
\begin{align}\label{eq:width_lower_bound_first_row}
     d_{2k}(W_1) \leq d_k(W).
\end{align}
To see this, denote by $u_i(1), u_i(2) \in \mathbb{R}^{2T}$ the first and second row of matrix $u_i$, respectively. The first row of $\mu(G)$ can be written as $\sum_{i=1}^k a_i(1,1) u_i(1) + a_i(1,2) u_i(2)$, a linear comibination of $2k$ row vectors, where $a_i(1,1), a_i(1,2)$ are the elements of the first row of matrix $a_i$.

To lower bound the generalized Kolmogorov width of the constrained set $W_1$, we consider a relaxed \emph{weighted} version of the width. Precisely, let $p$ be a probability measure on the set $W_1$, then the \emph{weighted squared deviation} of $W_1$ from $U$ under weight $p$ is defined as

\begin{align}\label{eq:width_lower_bound_by_weighted} 
    d_{2k}^2(W_1;p) \triangleq \inf_{U \in \mathcal{U}_{2k}} \mathbb{E}_{x\sim p} \inf_{y\in U}\|x-y\|^2 \leq \inf_{U \in \mathcal{U}_{2k}} \sup_{x \in W_1} \inf_{y\in U}\|x-y\|^2 = d^2_{2k}(W_1). 
\end{align}
We observe that $d_{2k}^2(W_1;p)$ in general can be computed using spectral methods. Indeed, for the subset $U$, the $y$ that achieves $\inf_{y\in U} \|x-y\|^2$ can be computed via a projection matrix $\hat{y} = U_{2k}U_{2k}^\top x$, where $U_{2k}$ consists of $2k$ columns of orthonormal vectors. We now have 
\begin{align}\notag 
    \mathbb{E}_{x\sim p} \inf_{y\in Q} \|x-y\|^2 & = \mathbb{E}_{x\sim p} \|x-U_{2k} U_{2k}^\top x\|^2 \\ \notag 
    & = \mathbb{E}_{x\sim p}[x^\top x - x^\top U_{2k} U_{2k}^\top x]    \\ \notag 
    & = \text{tr}((I- U_{2k} U_{2k}^\top)\mathbb{E}_{x\sim p}[xx^\top]).
\end{align}
The minimizer $U_{2k}$ of $\text{tr}((I- U_{2k} U_{2k}^\top)\mathbb{E}_{x\sim p}[xx^\top])$ is the same as the maximizer of $\text{tr}(U_{2k} U_{2k}^\top\mathbb{E}_{x\sim p}[xx^\top])$, which is given by the first $2k$ eigenvectors of $\mathbb{E}_{x\sim p}[xx^\top]$, and the value of the weighted squared generalized $k$-width is given by the sum of all eigenvalues of $\mathbb{E}_{x\sim p}[xx^\top]$ except for the first largest $2k$ eigenvalues (Lemma \ref{lemma:average_pca_error}).

We compute the weighted squared generalized $k$-width of the constrained set $W_1$, and it would serve as a lower bound of the squared generalized $k$-width. We choose the probability measure of $(a,b)^\top\in \mathbb{R}^2$ as the uniform measure on the unit circle. We compute the matrix $\mathbb{E}_{x\sim p}[xx^\top]$, which is $\mathbb{E}[\mu_1(G)^\top \mu_1(G)]$, where $\mu_1(G)$ is the first row of $\mu(G)$. Concretely, we write $\mu_1(G) = [\nu_0;\nu_1;\ldots;\nu_{T-1}]$ for $\nu_l \in \mathbb{R}^2$ and equal to
\begin{align*}
    \nu_l = [\Re(z^l), \Im(z^l)],
\end{align*}
where $z = a+jb$ and for all $l \in \{0, 1, \dots, T-1\}$. 

We claim that $\mathbb{E}[\nu_l \nu_m^\top] =0$ whenever $l \neq m$. Indeed, when $l\neq m$, each of the $4$ entries of matrix $\mathbb{E}[\nu_l \nu_m^\top]$ are of the form either $\Re(z^l)\Re(z^m), \Im(z^l)\Im(z^m)$, or $\Re(z^l)\Im(z^m)$ for some $l\neq m$. For the complex number $z = re^{j\theta}$, we know $z^l = r^l e^{jl\theta}$ which implies that $\Re(z^l) = r^l \cos(l\theta)$ and $\Im(z^l) = r^l \sin(l\theta)$. We now compute $\mathbb{E}[r^l \cos(l\theta) r^m \sin(m\theta)]$ for $l\neq m, l\geq 0, m\geq 0$ and other cases can be computed analogously. Since we are considering a uniform distribution on the unit circle, $r \equiv 1$. We have
\begin{align}\notag 
    & \int_{\theta\in [0,2\pi]}  \cos(k\theta) \sin(m\theta) \frac{1}{2\pi} d\theta \nonumber \\ \notag 
    & = \frac{1}{2\pi} \int_{0}^{2\pi} \frac{1}{2}(\sin((k+m)\theta) + \sin((k-m)\theta))d\theta \\ \notag 
    & = 0. 
\end{align}
Hence, it suffices to only compute $\E[\nu_l \nu_l^\top]$ 
\begin{align}
    \mathbb{E}[\nu_l \nu_l^\top] = \frac{1}{2} \begin{bmatrix} 1 & 0 \\ 0 & 1\end{bmatrix}.
\end{align}
Therefore, $\E[\mu_1(G)^\top \mu_1(G)] = 0.5 I_{2T}$. Using Lemma \ref{lemma:average_pca_error} $d^2_{2k}(W_1;p = \mathcal{U})$ is equal to the sum of bottom $2T- 2k$ eigenvalues: $d_{2k}^2(W_1;p = \mathcal{U}) = (2T-2k)/2 = T-k$. By \eqref{eq:width_lower_bound_first_row} and \eqref{eq:width_lower_bound_by_weighted}
\begin{align*}
    d_{k}(W) \geq d_{2k}(W_1) \geq d_{2k}(W_1;p= \mathcal{U}) = \sqrt{T - k}.
\end{align*}
\end{proofkwidth}

\newpage 

\section{Convex relaxation analysis: Proof of Theorem \ref{thm:convex_relaxation_informal}}\label{sec:convex_relaxation_bias}
Recall matrix $\widetilde{\Theta}$ defined in \eqref{eq:pca_parameters}. The following theorem is a formal restatement of Theorem \ref{thm:convex_relaxation_informal} that analyzes the approximation error due to convex relaxation. 
\begin{customthm}{\ref{thm:convex_relaxation_informal}}\label{thm:convex_relaxation_bound}  \normalfont{\textbf{(Convex relaxation error bound)}} 
\textit{Denote by $m_t$, the one-step-ahead predictions made by the best linear predictor (Kalman filter) for system (4). Let $R_P = \max\{\|B\|_2, \|C\|_2, \|D\|_2 \}$, $R_C = \max \{\|P\|_2, \|Q\|_2, \|R\|_2, \|K\|_2\}$, and $\|x_t\|_2 \leq R_x$. Suppose that $\|A^t\|_2 \leq \gamma t^{\log (\gamma)}$ for a bounded constant $\gamma \geq 1$. Let $C_0 = 43, C_1 = 520$. For any $\epsilon, \delta > 0$, if the number of filters $k$ satisfies
\begin{align*}
    k \geq \frac{\pi^2}{8} \log(T) \log \Big(\frac{12C_0d^2 (1+R_P^2)^3(2R_x^2 +R_C)(1+R_C^2) (1+\gamma)^4 (mT +\log(1/\delta)) T^{3+2\log(\gamma)}}{\epsilon} \Big),
\end{align*}
then the following holds for $\widetilde{\Theta}$
\begin{align}\label{eq:convex_relaxation_assumption_on_k}
    \Prob\Big[ \|\widetilde{\Theta} f_t - m_t\|_2^2 \geq \epsilon \Big] \leq \delta.
\end{align}}
\end{customthm}
\begin{proof}
Denote by $G = U \Lambda U^{-1}$ the eigendecomposition of matrix $G$, where $\Lambda = \mathsf{diag} (\lambda_1, \dots \lambda_d)$ are eigenvalues of $G$. Let $v_l$ be the columns of $U$ and $w^\top _l$ be rows of $U^{-1}$. Write 
\begin{align*}
    m_t = & \sum_{i=1}^{t-1} C G^{t-i-1} K y_i + \sum_{i=1}^{t-1} C G^{t-i-1} (B-KD) x_i + D x_t\\
    = & \sum_{i=1}^{t-1} C U \Lambda^{t-i-1} U^{-1} K y_i + \sum_{i=1}^{t-1} C U \Lambda^{t-i-1} U^{-1} (B-KD) x_i + D x_t\\
    = & \sum_{i=1}^{t-1} CU \Big[ \sum_{l=1}^d (\lambda_l^{t-i-1})e_l \otimes e_l \Big]  U^{-1}K y_i + \sum_{i=1}^{t-1} CU \Big[ \sum_{l=1}^d (\lambda_l^{t-i-1})e_l \otimes e_l \Big]U^{-1} (B-KD) x_i+ D x_t\\
     = & \sum_{l=1}^d CU e_l \otimes e_l U^{-1}K \sum_{i=1}^{t-1} \lambda_l^{t-i-1} y_i + \sum_{l=1}^d  CU  e_l \otimes e_lU^{-1} (B-KD) \sum_{i=1}^{t-1} \lambda_l^{t-i-1} x_i+ D x_t\\
     = & \sum_{l=1}^d C v_l w^\top _l K \sum_{i=1}^{t-1} \lambda_l^{t-i-1} y_i + \sum_{l=1}^d   Cv_l w^\top _l (B-KD) \sum_{i=1}^{t-1} \lambda_l^{t-i-1} x_i + D x_t.
\end{align*}
Let $Y_t = [y_1, \dots, y_{t}] \in \mathbb{R}^{m \times t}$ and $X_t = [x_1, \dots, x_t] \in \mathbb{R}^{n \times t}$. We can write $m_t$ and $\widetilde{m}_t$ as
\begin{align*}
m_t = & \sum_{l=1}^d C v_l w^\top _l K Y_{t-1}  \mu_{t-1:1}(\lambda_l)
    + \sum_{l=1}^d   Cv_l w^\top _l (B-KD) X_{t-1}  \mu_{t-1:1}(\lambda_l) + Dx_t,\\
\widetilde{m}_t = & \sum_{l=1}^d C v_l w^\top _l K Y_{t-1}  \widetilde{\mu}_{t-1:1}(\lambda_l)
    + \sum_{l=1}^d   Cv_l w^\top _l (B-KD) X_{t-1}  \widetilde{\mu}_{t-1:1}(\lambda_l) + Dx_t.
\end{align*}
We write $b_t = m_t - \widetilde{m}_t$ using the PCA reconstruction error $r_t = \mu_t - \widetilde{\mu}_t$
\begin{align*}
    b_t = m_t - \widetilde{m}_t = \sum_{i=1}^d Cv_i w_i^\top  K Y_{t-1} r_{t-1:1}(\lambda_i) + Cv_i w_i^\top  (B-KD) X_{t-1} r_{t-1:1}(\lambda_i).
\end{align*}
The Euclidean norm of bias is bounded by
\begin{align*}
    \|b_t\|_2 & \leq \Big( \sum_{i=1}^d \|C\|_2 \|v_i w_i^\top\|_2 \|K\|_2 \|Y_{t-1}\|_2 + \|C\|_2 \|v_i w_i^\top\|_2 (\|B\|_2 + \|K\|_2 \|D\|_2) \|X_{t-1}\|_2 \Big) \sup_{\lambda} \|r(\lambda)\|_2\\
    & \leq \Big( d R_P R_C \|Y_{t-1}\|_2  + d R_P^2(1+R_C) \|X_{t-1}\|_2 \Big)\sup\limits_{\lambda} \|r(\lambda)\|_2\\
    & \leq \Big( d R_P R_C \|Y_{t-1}\|_2 +  d R_P^2(1+R_C) \sqrt{t}R_x \Big) \Big(C_0 T \sqrt{\log T} c^{-k/\log T}\Big)^{1/2}
\end{align*}
The first inequality uses simple properties use as sub-multiplicative and sub-additive properties of norm. The second inequality uses the upper bound assumptions on parameters. The third inequality is due to \ref{lemma:bounding_mu_approximation_error} where $c = \exp(\pi^2/8)$ and $C_0 = 43$. The squared approximation error is given by 
\begin{align*}
    \|b_t\|^2_2 \leq 2d^2(1+R_C^2)(1+R^2_P)^2 \Big(\|Y_{t-1}\|^2_2 + t R^2_x \Big) C_0 T \sqrt{\log T} c^{-k/\log T}
\end{align*}
Observe that $\|Y_{1:t}\|_2^2 \leq \|Y_{1:t}\|_F^2 = \|y_{1:t}\|_2^2$. By \eqref{ineq:bound_inequality_observation_norm}, the following holds with probability greater than $1-\delta$
\begin{align*}
    & \|b_t\|_2^2
    \leq 12d^2 (1+R_P^2)^3(2R_x^2 +R_C)(1+R_C^2) (1+\gamma)^4 (mT +\log(1/\delta)) T^{3+2\log(\gamma)}c^{-k/\log T}
\end{align*}
We finish the proof by setting the number of filters $k$ such that the error is smaller than $\epsilon$, i.e.
\begin{align*}
    k \geq \frac{\log T}{\log c} \log \Big(\frac{12C_0d^2 (1+R_P^2)^3(2R_x^2 +R_C)(1+R_C^2) (1+\gamma)^4 (mT+\log(1/\delta)) T^{3+2\log(\gamma)}}{\epsilon} \Big).
\end{align*}
\end{proof}
Informally, the above theorem states that choosing $k \asymp_M \log(T) \log(T/\epsilon)$ is sufficient to ensure an approximation error smaller than $\epsilon$.

\section{Regret decomposition}\label{app:regret_decomposition}
Recall the definitions of innovation $e_t$ and model bias $b_t$
\begin{align}\label{eq:innovation_bias_definitions}
    e_t = y_t - \E[y_t|y_{1:t-1},x_{1:t}] = y_t - m_t  \qquad \text{and} \qquad b_t = \widetilde{\Theta}f_t - m_t = \widetilde{m}_t - m_t,
\end{align}
where $m_t$ is the predictions made by the Kalman filter in hindsight and $\widetilde{\Theta}$ is defined in \eqref{eq:pca_parameters}.
Let $\hat{m}_t$ be the predictions made by the algorithm. Regret can be written as 
\begin{align}\notag 
    \text{Regret}(T) & = \sum_{t=1}^T \|y_t - \hat{m}_t\|_2^2 - \|y_t - m_t\|_2^2 \\ \notag 
    & =  \sum_{t=1}^T \|m_t + e_t - \hat{m}_t\|_2^2- \|e_t\|_2^2\\ \notag 
    & = \sum_{t=1}^T \|\hat{m}_t - m_t\|_2^2 - \sum_{t=1}^T 2e_t^\top (\hat{m}_t - m_t)\\ \notag 
    & = \cL(T) - \sum_{t=1}^T 2e_t^\top (\hat{m}_t - m_t),
\end{align}
where $\cL(T)$ is the squared error between the Kalman filter predictions and algorithm predictions defined in \eqref{eq:risk_definition_L_t}. Recall the following notation
\begin{align*}
    Z_t \triangleq \alpha I + \sum_{i=1}^t f_if_i^\top, \qquad E_t \triangleq \sum_{i=1}^t e_i f_i^\top, \qquad  B_t \triangleq \sum_{i=1}^{t} b_i f_i^\top.
\end{align*}
The error between the predictions made by our algorithm and Kalman filter can be written as  
\begin{align*}
    \hat{m}_t - m_t =  \hat{\Theta}^{(t)}f_t - \widetilde{\Theta}f_t + b_t = \Big( \sum_{i=1}^{t-1} y_i f_i^\top  \Big)Z_{t-1}^{-1} f_t - \widetilde{\Theta}f_t + b_t,
\end{align*}
The second equation uses the update rule of $\hat{\Theta}^{(t)}$ given in \eqref{eq:theta_update_rule}. Simple algebraic manipulations give
\begin{align*}
    & \hspace{5mm} \Big( \sum_{i=1}^{t-1} y_i f_i^\top  \Big)Z_{t-1}^{-1} f_t - \widetilde{\Theta}f_t + b_t\\
    & = \Big( \sum_{i=1}^{t-1} [\widetilde{\Theta}f_i + b_i + e_i] f_i^\top  \Big)Z_{t-1}^{-1} f_t - \widetilde{\Theta}f_t + b_t\\
    & = \Big( \sum_{i=1}^{t-1} [\widetilde{\Theta}f_if_i^\top  + b_if_i^\top  + e_if_i^\top ]  \Big)Z_{t-1}^{-1} f_t - \widetilde{\Theta}f_t + b_t\\
    & = \Big( \sum_{i=1}^{t-1} [\widetilde{\Theta}(f_if_i^\top  + \frac{\alpha}{t-1} I - \frac{\alpha}{t-1} I) + b_if_i^\top  + e_if_i^\top ]  \Big)Z_{t-1}^{-1} f_t - \widetilde{\Theta}f_t + b_t\\
    & = \widetilde{\Theta} \Big( \alpha I+\sum_{i=1}^{t-1} f_if_i^\top \Big) Z_{t-1}^{-1} f_t
    - \alpha \widetilde{\Theta} Z_{t-1}^{-1} f_t + \Big(\sum_{i=1}^{t-1} b_i f_i^\top  \Big)Z_{t-1}^{-1} f_t + \Big(\sum_{i=1}^{t-1} e_i f_i^\top  \Big)Z_{t-1}^{-1} f_t - \widetilde{\Theta}f_t + b_t\\
    & = \widetilde{\Theta} Z_{t-1} Z_{t-1}^{-1} f_t
    - \alpha \widetilde{\Theta} Z_{t-1}^{-1} f_t + B_{t-1} Z_{t-1}^{-1} f_t + E_{t-1} Z_{t-1}^{-1} f_t - \widetilde{\Theta}f_t + b_t\\
    &  = E_{t-1}Z_{t-1}^{-1} f_t
    + B_{t-1}Z_{t-1}^{-1} f_t + b_t - \alpha \widetilde{\Theta} Z_{t-1}^{-1} f_t.
\end{align*}
We apply the RMS-AM inequality to obtain an upper bound on $\cL(T)$
\begin{align*}
    \cL(T) & = \sum_{t=1}^T \|\hat{m}_t - m_t\|_2^2\\
    & = \sum_{t=1}^T \|E_{t-1}Z_{t-1}^{-1} f_t
    + B_{t-1}Z_{t-1}^{-1} f_t + b_t - \alpha \widetilde{\Theta} Z_{t-1}^{-1} f_t\|_2^2\\
    & \leq \sum_{t=1}^T 3\|E_{t-1}Z_{t-1}^{-1} f_t\|_2^2 + 3 \|B_{t-1}Z_{t-1}^{-1} f_t + b_t\|_2^2 + 3 \|\alpha \widetilde{\Theta} Z_{t-1}^{-1} f_t\|_2^2.
\end{align*}
Regret can thus be decomposed to the following terms
\begin{alignat*}{2}
    \text{Regret}(T) & \leq \sum_{t=1}^T 3\|E_{t-1}Z_{t-1}^{-1} f_t\|_2^2  \qquad && \text{(least squares error)}\\
    & + \sum_{t=1}^T 3 \|B_{t-1}Z_{t-1}^{-1} f_t + b_t\|_2^2 \qquad && \text{(improper learning bias)} \\
    & + \sum_{t=1}^T 3 \|\alpha \widetilde{\Theta} Z_{t-1}^{-1} f_t\|_2^2 \qquad && \text{(regularization error)}\\
    & - \sum_{t=1}^T 2e_t^\top (\hat{m}_t - m_t) \qquad && \text{(innovation error)}
\end{alignat*}
We bound each of the first three terms by extracting a $\|Z_{t-1}^{-1/2} f_t\|_2^2$, i.e. we write
\begin{align*}
    \|E_{t-1}Z_{t-1}^{-1} f_t\|_2^2 & \leq \sup_{1 \leq t \leq T} \|E_{t-1} Z_{t-1}^{-1/2}\|_2^2  \sum_{t=1}^T \|Z_{t-1}^{-1/2} f_t\|_2^2, \\
    \|B_{t-1}Z_{t-1}^{-1} f_t + b_t\|_2^2 & \leq  \sup_{1 \leq t \leq T} \|B_{t-1} Z_{t-1}^{-1/2}\|_2^2 \sum_{t=1}^T  \|Z_{t-1}^{-1/2} f_t\|_2^2+ \sum_{t=1}^T \|b_t\|_2^2,\\
    \|\alpha \widetilde{\Theta} Z_{t-1}^{-1} f_t\|_2^2 & \leq  \sup_{1 \leq t \leq T} \|\alpha \widetilde{\Theta} Z_{t-1}^{-1/2}\|_2^2  \sum_{t=1}^T \|Z_{t-1}^{-1/2} f_t\|_2^2.
\end{align*}

In subsequent sections, we compute a high probability upper bound on $\sum_{t=1}^T \|Z_{t-1}^{-1/2} f_t\|_2^2$ as well as the specific terms in the above decomposition that affect least squares error, improper learning bias, regularization error, and innovation error, proving that regret is bounded by $\polylog(T)$.

\section{Regret analysis}\label{app:regret_analysis}

\subsection{High probability bound on $\det(Z_t)$}\label{app:det_z_t_bound}
We start by deriving an upper bound on $\log (\det(Z_t))$ as this quantity appears multiple times when analyzing regret. The following lemma provides a high probability bound on $\det(Z_t)$ for features defined in \eqref{eq:feature_def}.
\begin{lemm}\label{lemma:high_probability_upper_bound_on_det_Z_t} \normalfont{\textbf{(High probability upper bounds on $\mathbf{\text{det}(Z_t))}$}} \textit{Assume as in Lemma \ref{lemma:observation_norm_high_prob_bound} and let $Z_t = \alpha I + \sum_{i=1}^t f_t f_t^\top$. Then, for any $\delta \geq 0$}
\begin{align*}
    \Prob \Big(\log (\det(Z_t)) \geq l \log \Big[ \alpha^2 + 8k (R_P^2+1)(R_x^2+R_C) (1+\gamma)^4 (mt+\log \Big( \frac{1}{\delta}\Big))) t^{3+2\log (\gamma)} \Big] \Big) \leq \delta.
\end{align*}
\end{lemm}
\begin{proof}
Let $l$ be the feature vector dimension. We have
\begin{align*}
    Z_t = \alpha I + \sum_{i=1}^t f_t f_t^\top  \preceq \alpha I + \sum_{i=1}^t (f_i^\top  f_i) I \quad \Rightarrow \quad \det (Z_t) \leq  \Big(\alpha^2 + \sum_{i=1}^t \|f_i\|_2^2\Big)^l.
\end{align*}
Recall the definition $\Psi_{t}=[\psi_t, \dots, \psi_1]$ from Algorithm \ref{alg:slip} and the compact representation for input features $\widetilde{x}_t = (\Psi_t \otimes I_n) x_{1:t}$ and output features $\widetilde{y}_t = (\Psi_t \otimes I_n) y_{1:t}$. Observe that $\|\Psi_t\|_2 \leq 1$ since $\Psi_t$ is a block of eigenvector matrix of hankel matrix $H$. Thus the feature norm is bounded by 
\begin{align*}
    \|f_t\|_2^2 & = \|\widetilde{y}_{t-1}\|_2^2 + \|\widetilde{x}_{t-1}\|_2^2 + \|x_t\|_2^2 \\
    & \leq k \|y_{1:t-1}\|_2^2 + k \|x_{1:t-1}\|_2^2 + R_x^2\\
    & \leq k \|y_{1:t}\|_2^2 + 2k t R_x^2.
\end{align*}
From Lemma \ref{lemma:observation_norm_high_prob_bound}, with probability at least $1 - \delta$
\begin{align*}
    \|f_t\|_2^2 & \leq 6 k (R_P^2+1)(R_x^2+R_C) (1+\gamma)^4 (mt+\log \Big( \frac{1}{\delta}\Big)) t^{2+2\log (\gamma)} + 2 k t R_x^2\\
    & \leq 8k (R_P^2+1)(R_x^2+R_C) (1+\gamma)^4 (mt+ \log \Big( \frac{1}{\delta}\Big)) t^{2+2\log (\gamma)}.
\end{align*}
The above bound is increasing in $t$, therefore
\begin{align*}
    \Prob \Big(\det(Z_t) \geq \Big[ \alpha^2 + 8k (R_P^2+1)(R_x^2+R_C) (1+\gamma)^4 (mt+\log \Big( \frac{1}{\delta}\Big)) t^{3+2\log (\gamma)} \Big]^l \Big) \leq \delta.
\end{align*}
\end{proof}
Given the PAC bound parameters $M$, if $k \asymp_M \polylog(T)$ (and hence $l = (m+n)k + n \asymp_M \polylog(T)$), then the above lemma states that $\log(\det(Z_t)) \lesssim_M \polylog(T)$.

\subsection{Self-normalizing vector martingales}\label{app:self_normalizing_martingales}
We now prove a key result on vector self-normalizing martingales that is used multiple times throughout our regret analysis. The result is inspired by Theorem 1 of \citet{abbasi2011improved}, which provides a bound for self-normalizing martingales with scalar sub-Gaussian noise, and extend it to vector-valued sub-Gaussian noise with arbitrary covariance.

\begin{theo} \label{thm:self_normalized_vector_matringale} \normalfont{\textbf{(Bound on self-normalized vector martingale)}} \textit{Let $\{\mathcal{F}_t\}_{t=0}^\infty$ be a filtration. Let $e_t \in \mathbb{R}^m$ be $\mathcal{F}_t$ measurable and $e_t |\mathcal{F}_{t-1}$ to be conditionally $R_V$-sub-Gaussian. In other words, for all $t \geq 0$ and $\omega \in \mathbb{R}^m$
\begin{align*}
    \E[\exp(\omega^\top e_t)\mid \mathcal{F}_{t-1}] \leq \exp(R_V^2\|\omega\|_2^2/2).
\end{align*}
Let $f_t \in \mathbb{R}^l$ be an $\mathcal{F}_{t-1}$-measurable stochastic process. Assume that $Z$ is an $l \times l$ positive definite matrix. For any $t \geq 0$, define}
\begin{align*}
    Z_t = Z_0 + \sum_{i=1}^t f_t f_t^\top \qquad \text{and} \qquad E_t = \sum_{i=1}^t e_i f_i^\top.
\end{align*}
\textit{Then, for any $\delta > 0$ and for all $t \geq 0$}
\begin{align*}
    \Prob\Big[\|E_t Z_t^{-1/2}\|_2 \leq 8R_V^2m + 4R_V^2 \log \Big( \frac{\det(Z_t)^{1/2} \det(Z_0)^{-1/2}}{\delta}\Big)  \Big] \geq 1-\delta.
\end{align*}
\end{theo}
\begin{proof}
We use an $\epsilon$-net argument. First, we establish control over $\|\omega^\intercal E_t Z_t^{-1/2}\|_2$ for all vectors $\omega$ in unit sphere $\mathcal{S}^{m-1}$. We will discretize the sphere using a net and finish by taking a union bound over all $\omega$ in the net.  

Let $\mathcal{N}$ be an $\epsilon$-net of unit sphere $\mathcal{S}^{m-1}$ and set $\epsilon = 1/2$. Corollary 4.2.13 in \citet{vershynin2018high} states that the covering number for unit sphere $\mathcal{S}^{m-1}$ is given by
\begin{align*}
    |\mathcal{N}| \leq \Big(\frac{2}{\epsilon}+1\Big)^m =5^m.
\end{align*}
 $\omega^\top e_i$ is $R_V$-sub-Gaussian for any $\omega \in \mathcal{N}$. Therefore, for any $\omega \in \mathcal{N}$ and any $u \geq 0$, Theorem 1 in \citet{abbasi2011improved} yields
\begin{align*}
    \Prob\Big[\|\omega^\top E_t Z_t^{-1/2}\|_2 \geq u \Big] \leq \det(Z_t)^{1/2} \det(Z_0)^{-1/2} \exp\big(-\frac{u}{2R_V^2}\big).
\end{align*}
Using Lemma 4.4.1 in \citet{vershynin2018high}, we have
\begin{align*}
    \|E_t Z_t^{-1/2}\|_2 \leq 2 \sup\limits_{\omega \in \mathcal{N}} \|\omega^\top E_t Z_t^{-1/2}\|_2.
\end{align*}
Taking a union bound over $\mathcal{N}$, we conclude that
\begin{align*}
    \Prob\big[\|E_t Z_t^{-1/2}\|_2 \geq u \big] & \leq \Prob\big[\sup\limits_{\omega \in \mathcal{N}} \|\omega^\top E_t Z_t^{-1/2}\|_2 \geq \frac{u}{2} \big]\\
    & \leq \sum\limits_{\omega \in \mathcal{N}} \Prob\Big[\|\omega^\top E_t Z_t^{-1/2}\|_2 \geq \frac{u}{2} \Big]\\
    & \leq \det(Z_t)^{1/2} \det(Z_0)^{-1/2} \exp\big(2m-\frac{u}{4R_V^2}\big).
\end{align*}
\end{proof}
The above theorem combined with the result of Lemma \ref{lemma:high_probability_upper_bound_on_det_Z_t} immediately implies that for $k \asymp_M \polylog(T)$, we have $\|E_t Z_t^{-1/2}\|_2 \lesssim_M \polylog(T)$ and $\|B_t Z_t^{-1/2}\|_2 \lesssim_M \polylog(T)$ with high probability.

\subsection{High probability bound on $\|Z_{t-1}^{-1/2} f_t\|_2^2$}\label{app:bound_Zt_inv_ft}

In this section, we show that$\sum_{t=1}^T \|Z_{t-1}^{-1/2}f_t\|_2^2 \lesssim_M \polylog(T)$. The proof steps are summarized below.
\begin{enumerate}[Step 1.]
    \item We show a high probability L\"owner upper bound on $f_tf^\top_t$ in terms of $\alpha_0 I + \E[f_t f_t^\top]$.
    \item We state the \textit{block-martingale small-ball condition} and show that the process $\{f_t\}$ satisfies this condition. We prove a high probability lower bound on $Z_t$ in terms of the conditional covariance $\cov(f_{s+i} \mid \mathcal{F}_i)$ for large enough $s$.
    \item We define a \textit{filter quadratic function condition} and prove that under this condition, there exists $c_T \asymp_M \polylog(T)$ such that $Z_t - \frac{1}{c_T} f_{t+1} f^\top_{t+1} \succeq 0$. By Schur complement lemma, this is equivalent to $\|Z_t^{-1/2} f_{t+1}\|_2 \leq c_T \asymp_M \polylog(T)$.
\end{enumerate}

\paragraph{Step 1.} The following lemma establishes a high probability upper bound on $f_tf_t^\top$ based on the covariance of feature vector $f_t$.
\begin{lemm}\label{lemma:high_probability_upper_bound_ftftT} \normalfont{\textbf{(High probability upper bound on $\mathbf{f_tf_t^\top}$)}} \textit{Let $f_t$ be a zero-mean Gaussian random vector in $\mathbb{R}^l$ and let $\Sigma_t = \alpha_0 I + \E[f_tf_t^\top]$ for a real $\alpha_0 > 0$. Then, for any $\delta > 0$ and $\alpha_0 > 0$
\begin{align*}
    \Prob \Big( f_tf_t^\top \preceq [2l + 4 \log(1/\delta)]\Sigma_t \Big) \geq 1- \delta,
\end{align*}
and if $\Sigma_t$ is invertible, the results holds for $\alpha_0 = 0$.}
\end{lemm}
\begin{proof}
Consider the random vector $\Sigma_t^{-1/2} f_t$. Jensen's inequality gives
\begin{align*}
    \E \|\Sigma_t^{-1/2} f_t\|_2 \leq \sqrt{\E[f_t^\top \Sigma^{-1}_t f_t]} = \sqrt{\tr(\Sigma^{-1}_tE[f_tf^\top_t])} \leq \sqrt{l}.
\end{align*}
By standard bounds on tails of sub-gaussian random variables (for example, see Exercise 6.3.5 in \citet{vershynin2018high}), for any $\delta > 0$ 
\begin{align*}
    \Prob \Big(\|\Sigma_t^{-1/2} f_t\|_2 > \sqrt{l} + \sqrt{2 \log \frac{1}{\delta}} \Big) \leq \delta 
\end{align*}
Let $c =  2l + 4 \log \frac{1}{\delta}$. Then, the above bound implies 
\begin{align*}
    \Prob(f_t^\top \Sigma_t^{-1} f_t  \leq c) \geq 1-\delta. 
\end{align*}
Using Schur complement method, $c - f_t^\top \Sigma_t^{-1} f_t \geq 0$ if and only if the following matrix is positive semi-definite
\begin{align*}
    \begin{bmatrix}
    \Sigma_t & f_t\\
    f_t^\top  & c
    \end{bmatrix} \succeq 0.
\end{align*}
Using the other Schur complement, this is only true if and only if $\Sigma_t - \frac{1}{c}f_tf_t^\top \succeq 0$, which concludes the proof.
\end{proof}

\paragraph{Step 2.}
To capture the excitation behavior of features, we use the martingale small-ball condition \cite{mendelson2014learning, simchowitz2018learning}.
\begin{definition} \normalfont{\textbf{(Martingale small-ball)}} \textit{Let $\{f_t\}_{t \geq 1}$ be an $\mathcal{F}_t$-adapted random processes taking values in $\mathbb{R}^l$. We say that $\{f_t\}_{t \geq 1}$ satisfies the $(s, \Gamma_{\text{sb}}, p)$-block martingale small-ball (BMSB) condition for $\Gamma_{sb} \succ 0$ if for any $t \geq 1$ and for any fixed $\omega$ in unit sphere $\mathcal{S}^{l-1}$}
\begin{align*}
    \frac{1}{s}\sum_{i=1}^s \mathbb{P}(|w^\top f_{t+i}|\geq \sqrt{w^\top \Gamma_{\text{sb}} w} \mid \mathcal{F}_t) \geq p.
\end{align*}
\end{definition}

To show the process $\{f_t\}_{t \geq 1}$ satisfy a BMSB condition, we first  show that the conditional covariance of features is increasing in the positive semi-definite cone. 

\begin{lemm}\label{lemma:covariance_is_increasing} \normalfont{\textbf{(Monotonicity of conditional covariance of features)}} \textit{Let $\phi_1, \dots, \phi_k$ for $k \leq T$ be a set of $T$-dimensional orthogonal vectors and let $\psi_i = [\phi_1(i), \dots, \phi_k(i)]^\top$ be a $k$-dimensional vector. Consider system (4) and define the following for all $t \geq 2$
\begin{align}\label{eq:feature_def_increasing_cov_lemma}
    f_t = \psi_1 \otimes y_{t-1} + \dots + \psi_{t-1} \otimes y_1.
\end{align}
Let $\mathcal{F}_t = \sigma\{\eta_0,\dots, \eta_{t-1}, \zeta_1, \dots, \zeta_t\}$. 
Then, $\cov(f_{t+i}|\mathcal{F}_t)$ is independent of $t$ and increases with $i$ in the positive semi-definite cone.}
\end{lemm}
\begin{proof}
Expanding $y_i$ in definition of $f_t$ in \eqref{eq:feature_def_increasing_cov_lemma} based on system \eqref{eq:lds_definition}, we have
\begin{align*}
    f_{t+i} - \E[f_{t+i} \mid \mathcal{F}_t] & = (\psi_1 \otimes C) \eta_{t+i-2}\\
    & + (\psi_2 \otimes C + \psi_1 \otimes CA) \eta_{t+i-3}\\
    & + \dots \\
    & + (\psi_{i-1} \otimes C + \dots + \psi_1 \otimes CA^{i-2})\eta_{t}\\
    & + \psi_1 \otimes \zeta_{t+i-1} + \dots + \psi_{i-1} \otimes \zeta_{t+1}
\end{align*}
Recall that $\E[\eta_t \eta_t^\top] = Q, \E[\zeta_t \zeta_t^\top] = R$ and that the process noise and the observation noise are i.i.d. Therefore,
\begin{align}\label{eq:conditional_cov_ft}
    \begin{split}
        \cov(f_{t+i}|\mathcal{F}_t) & = (\psi_1 \otimes C) Q (\psi_1 \otimes C)^\top\\
        & + (\psi_2 \otimes C + \psi_1 \otimes CA) Q (\psi_2 \otimes C + \psi_1 \otimes CA)^\top\\
        & + \dots \\
        & + (\psi_{i-1} \otimes C + \dots + \psi_1 \otimes CA^{i-2}) Q (\psi_{i-1} \otimes C + \dots + \psi_1 \otimes CA^{i-2})^\top \\
        & + \psi_1 \otimes R \psi_1^\top \otimes I_m + \dots + \psi_{i-1} \otimes R \psi_{i-1}^\top \otimes I_m.
    \end{split}
\end{align}
Observe that the conditional covariance is independent of $t$. Furthermore, all terms in the above sum are positive semi-definite; increasing $i$ only adds two additional positive semi-definite terms. It follows that 
\begin{align*}
    \cov(f_{t+i+1}|\mathcal{F}_t) \succeq \cov(f_{t+i}|\mathcal{F}_t).
\end{align*}
\end{proof}
Equipped with the result of the above lemma, we now show that $\{f_t\}_{t \geq 1}$ satisfy a BMSB condition.

\begin{lemm}\label{lemma:small_ball_condition_ft} \normalfont{\textbf{(BMSB condition)}} \textit{Consider the process $\{f_t\}_{t \geq 1}$ defined in Lemma \ref{lemma:covariance_is_increasing} and let $\Gamma_i = \cov(f_{t+i}|\mathcal{F}_t)$. For any $1 \leq s \leq T$, the process $\{f_t\}_{t \geq 1}$ satisfies the $(s, \Gamma_{s/2}, 3/20)$-BMSB condition.}
\end{lemm}
\begin{proof}
Note that $\omega^\top f_{t+i}\mid \mathcal{F}_t$ has a Gaussian distribution with variance $\sqrt{\omega^\top \Gamma_i \omega}$. By an application of Paley-Zygmund inequality, one has
\begin{align*}
    \mathbb{P}(|w^\top f_{t+i}|\geq \sqrt{w^\top \Gamma_{i} w} \mid \mathcal{F}_t)  \geq \mathbb{P}(|w^\top f_{t+i} - \E[w^\top f_{t+i}\mid \mathcal{F}_t]|\geq \sqrt{w^\top \Gamma_{i} w} \mid \mathcal{F}_t) \geq \frac{3}{10}
\end{align*}
Let $1 \leq s' \leq s$. By Lemma \ref{lemma:covariance_is_increasing}, $\Gamma_i$ is increasing in $i$. Therefore,
\begin{align*}
    \frac{1}{s}\sum_{i=1}^s \mathbb{P}(|w^\top f_{t+i}|\geq \sqrt{w^\top \Gamma_{s'} w} |\mathcal{F}_t) & \geq \frac{1}{s}\sum_{i=s'}^s \mathbb{P}(|w^\top f_{t+i}|\geq \sqrt{w^\top \Gamma_{s'} w} |\mathcal{F}_t) \\
    & \geq \frac{1}{s}\sum_{i=s'}^s \mathbb{P}(|w^\top f_{t+i}|\geq \sqrt{w^\top \Gamma_{i} w} |\mathcal{F}_t)  \quad \text{($\Gamma_i$ increasing)}\\
    & \geq \frac{3}{10}\frac{s-s'+1}{s}. \hspace{3.5cm} \text{(Paley-Zygmund)}
\end{align*}
Choosing $s' = s/2$ shows that $f_t$ satisfies $(s, \Gamma_{s/2}, 3/20)$ small-ball condition.
\end{proof}

The small-ball condition can be used to establish high probability lower bound on $\sigma_{\min} (Z_t)$, as shown by the following lemma.
\begin{lemm}\label{lemma:lower_bound_Z_t} \normalfont{\textbf{(Lower bound on $\mathbf{Z_t}$)}} \textit{Consider the process $\{f_t\}_{t \geq 1}$ defined in Lemma \ref{lemma:covariance_is_increasing} and let $Z_t = \alpha I + \sum_{i=1}^t f_i f_i^\top$ for regularization parameter $\alpha > 0$. For $\delta, \alpha_0 > 0$ let}
\begin{align*}
    \Gamma_i = \cov(f_{t+i}|\mathcal{F}_i), \quad 
    \Gamma_{\max} = t [2l + 4 \log(2/\delta)][\alpha_0 I + \Gamma_t].
\end{align*}
\textit{For any $\delta > 0$ if $s$ satisfies the following}
\begin{align*}
    s \leq \frac{tp^2/10}{\log \det(\Gamma_{\max}) - l \log (\alpha) - \log(2/\delta)}, 
\end{align*}
\textit{then}
\begin{align*}
    \Prob \Big(Z_t \succeq  \frac{\alpha}{2} I + \frac{s \lfloor t/s \rfloor p^2  \Gamma_{s/2}}{16} \Big) \geq 1-\delta .
\end{align*}
\end{lemm}
\begin{proof}
According to Lemma \ref{lemma:small_ball_condition_ft}, $\{f_t\}_{t \geq 1}$ satisfies the $(s, \Gamma_{s/2}, p = 3/20)$-BMSB condition. The following lemma from \citet{simchowitz2018learning} gives tail probabilities for real-valued processes that satisfy a small-ball condition. Note that our notation for small ball condition in real-valued processes slightly differs from \citet{simchowitz2018learning} which results in a slight difference in the statement of the lemma below.
\begin{lemm}\label{lemma:small_ball_tail} \normalfont{\textbf{(Tail bounds for small-ball processes)}} \textit{If a real-valued process $\{z_t\}_{t \geq 1}$ satisfies the $(s,\sigma, p)$-BMSB condition, then}
\begin{align*}
    \Prob (\sum_{i=1}^t z_i^2 \leq \frac{p^2 \sigma}{8}s\lfloor t/s \rfloor) \leq \exp\Big(-\frac{\lfloor t/s \rfloor p^2}{8}\Big).
\end{align*}
\end{lemm}
For a fixed $\omega \in \mathcal{S}^{l-1}$, the process $\{\omega^\top f_t\}_{t \geq 1}$ satisfies $(s, \omega^\top \Gamma_{s/2} \omega, p)$. Using the above lemma, we have 

\begin{align*}
    \Prob \Big(\omega^\top \big( \sum_{i=1}^t f_i f_i^\top \big) \omega \leq \frac{p^2 \omega^\top \Gamma_{s/2} \omega }{8} s \lfloor t/s \rfloor \Big) \leq \exp\Big(-\frac{\lfloor t/s \rfloor p^2}{8}\Big).
\end{align*}
For large enough $t$, we can convert this high probability bound to obtain a uniform L\"owner lower bound on $Z_t$ by a discretization argument. 

Given a regularization parameter $\alpha > 0$, define
\begin{align*}
    \Gamma_{\min} & = \alpha I + \frac{s \lfloor t/s \rfloor p^2  \Gamma_{s/2}}{8}
\end{align*}
Define the following events
\begin{align*}
    \mathcal{E}_1 = \Big\{ Z_t \succeq \frac{\Gamma_{\min}}{2} \Big\} \quad \text{and} \quad 
    \mathcal{E}_2 = \Big\{ Z_t \preceq \Gamma_{\max} \Big\}.
\end{align*}
We have $\Prob(\mathcal{E}^c_1 ) \leq \Prob(\mathcal{E}^c_1 \cap \mathcal{E}_2) + \Prob(\mathcal{E}^c_2)$, where $\Prob(\mathcal{E}^c_2)$ is bounded by $\delta/2$ according to Lemma \ref{lemma:high_probability_upper_bound_ftftT}. Let $\mathcal{S}_{\Gamma_{\text{sb}}} = \{\omega: \omega^\top  \Gamma_{\text{sb}} \omega = 1\}$ and let $\mathcal{T}$ be a $1/4$-net of $\mathcal{S}_{\Gamma_{\text{sb}}}$ in the norm $\|\Gamma^{1/2}_{\max}(.)\|_2$. By Lemma 4.1 and Lemma D.1 in \citet{simchowitz2018learning}, we can write
\begin{align*}
    \Prob(\mathcal{E}^c_1 \cap \mathcal{E}_2) & =  \Prob \Big( \Big\{ Z_t \nsucceq \frac{\Gamma_{\min}}{2} \Big\} \cap \Big\{ Z_t \preceq \Gamma_{\max} \Big\} \Big)\\
    & \leq \Prob \Big( \Big\{ \exists \omega \in \mathcal{T}: \|Z_t\omega\|^2 <  \omega^\top \Gamma_{\min} \omega \Big\} \cap \Big\{ Z_t \preceq \Gamma_{\max} \Big\} \Big)\\
    & \leq \exp \Big(-\frac{\lfloor t/s\rfloor p^2}{8} + \log \det(\Gamma_{\max} \Gamma_{\min}^{-1})\Big)\\
    & \leq \exp \Big(-\frac{t p^2}{10s}+ \log \frac{\det(\Gamma_{\max})}{\alpha^l}\Big)
\end{align*}
Setting $s$ such that the above probability is bounded by $\delta/2$
\begin{align*}
    s \leq \frac{tp^2/10}{\log \det(\Gamma_{\max}) - l \log (\alpha) + \log(2/\delta)},
\end{align*}
we conclude that  $\Prob(\mathcal{E}^c_1 ) \leq \delta/2 + \delta/2 = \delta$.
\end{proof}

\paragraph{Step 3.} So far we have computed a lower bound on $Z_t$ and an upper bound on $f_tf_t^\top$ and our goal is to show that there exists $c_T \asymp_M \polylog(T)$ such that $Z_t - \frac{1}{c_T} f_tf_t^\top \succeq 0$. This inequality, however, does not hold for any set of orthonormal filters $\phi_1, \dots, \phi_k$. We identify an assumption connecting filters with transition matrix $A$ that ensures $Z_t - \frac{1}{c_T} f_tf_t^\top \succeq 0$. This assumption is based on a \textit{filter quadratic function}, which we restate below.

\begin{definition}\normalfont{\textbf{(Filter quadratic function)}} Let $\phi_1, \dots, \phi_k$ for $k \leq T$ be a set of $T$-dimensional vectors, let $\psi_i = [\phi_1(i), \dots, \phi_k(i)]^\top$ be a $k$-dimensional vector, and let $\psi_i^{(d)} = \psi_i \otimes I_d$, for any $d \geq 1$. For any matrix $A \in \mathbb{R}^{d \times d}$, the following matrix is called the \textit{filter quadratic function} of $\psi$ with respect to $A$
\begin{align*}
    \Omega_t(A; \psi) & = (\psi^{(d)}_1) (\psi^{(d)}_1)^\top + (\psi^{(d)}_2 + \psi^{(d)}_1 A) (\psi^{(d)}_2 + \psi^{(d)}_1 A)^\top + \dots\\
    & + (\psi^{(d)}_{t-1} + \dots + \psi^{(d)}_1 A^{t-2}) (\psi^{(d)}_{t-1}  + \dots + \psi^{(d)}_1 A^{t-2})^\top.
\end{align*}
\end{definition}

In the following lemma, we show that a condition on filter quadratic function implies $t \Gamma_{s/2} - \Gamma_{t+1}/c_0 \succeq 0$ for a constant $c_0$.
\begin{lemm}\normalfont{\textbf{(Filter quadratic condition)}} \label{lemma:filter_quadratic_condition} 
\textit{Assume as in Lemma \ref{lemma:covariance_is_increasing} and let $\kappa$ be the maximum condition number of $Q$ and $R$. For any $A$, if there exists $t_0 \geq 1$ for which there exists $s$ such that 
\begin{align*}
    t \Omega_{s/2}(A; \psi) - \Omega_{t+1}(A; \psi) \succeq 0, \qquad \forall t \geq t_0,
\end{align*}
then $t \Gamma_{s/2} - \Gamma_{t+1}/c_0 \succeq 0$, where $c_0 \geq \kappa$.}
\end{lemm}
\begin{proof}
Let $\psi^{(m)}_i = \psi_i \otimes I_m$. Recall the expression of the conditional covariance of $f_t$ given in \eqref{eq:conditional_cov_ft}:
\begin{align*}
    \Gamma_t & = (\psi^{(m)}_1 C) Q (\psi^{(m)}_1 C)^\top\\
    & + (\psi^{(m)}_2 C + \psi^{(m)}_1 CA) Q (\psi^{(m)}_2 C + \psi^{(m)}_1 CA)^\top\\
    & + \dots \\
    & + (\psi^{(m)}_{t-1} C + \dots + \psi^{(m)}_1 CA^{t-2}) Q (\psi^{(m)}_{t-1} C + \dots + \psi^{(m)}_1 CA^{t-2})^\top \\
    & + \psi^{(m)}_1 R (\psi^{(m)}_1)^\top + \dots + \psi^{(m)}_{t-1} R (\psi^{(m)}_{t-1})^\top
\end{align*}
Define the following terms 
\begin{align*}
    \Gamma^{(Q)}_t & \triangleq (\psi^{(m)}_1 C) Q (\psi^{(m)}_1 C)^\top + \dots + (\psi^{(m)}_{t-1} C + \dots + \psi^{(m)}_1 CA^{t-2}) Q (\psi^{(m)}_{t-1} C + \dots + \psi^{(m)}_1 CA^{t-2})^\top, \\
    \Gamma^{(R)}_t & \triangleq \psi^{(m)}_1 R (\psi^{(m)}_1)^\top + \dots + \psi^{(m)}_{t-1} R (\psi^{(m)}_{t-1})^\top,
\end{align*}
where $\Gamma_t = \Gamma^{(Q)}_t+ \Gamma^{(R)}_t$. In order to show $t \Gamma_{s/2} - \Gamma_{t+1}/c_0 \succeq 0$, it is sufficient to show 
\begin{align*}
    t \Gamma^{(Q)}_{s/2} - \frac{1}{c_0} \Gamma^{(Q)}_{t+1} \succeq 0 \quad \text{and} \quad t \Gamma^{(R)}_{s/2} - \frac{1}{c_0} \Gamma^{(R)}_{t+1} \succeq 0.
\end{align*}
Let $R_C = \max \{\|R\|_2, \|Q\|_2\}$ and $\sigma_r = \min\{\sigma_{\min}(Q), \sigma_{\min}(R)\}$. For $t \Gamma^{(R)}_{s/2} - \frac{1}{c_0} \Gamma^{(R)}_{t+1}$, we have 
\begin{align*}
    & \sigma_r [\psi^{(m)}_1 (\psi^{(m)}_1)^\top + \dots + \psi^{(m)}_{t} (\psi^{(m)}_{t})^\top] \\
    \preceq & \psi^{(m)}_1 R (\psi^{(m)}_1)^\top + \dots + \psi^{(m)}_{t} R (\psi^{(m)}_{t})^\top \\
    \preceq & R_C [\psi^{(m)}_1 (\psi^{(m)}_1)^\top + \dots + \psi^{(m)}_{t} (\psi^{(m)}_{t})^\top].
\end{align*}
Setting $c_0 = R_C/\sigma_r$, gives
\begin{align*}
    & t \Gamma^{(R)}_{s/2} - \frac{1}{c_0} \Gamma^{(R)}_{t+1}\\ 
    \succeq & \sigma_r t[\psi^{(m)}_1 (\psi^{(m)}_1)^\top + \dots + \psi^{(m)}_{s/2-1} (\psi^{(m)}_{s/2-1})^\top] - \sigma_r [\psi^{(m)}_1 (\psi^{(m)}_1)^\top + \dots + \psi^{(m)}_{t} (\psi^{(m)}_{t})^\top] \succeq 0.
\end{align*}
The last matrix is positive semi-definite based on assumption \eqref{eq:filter_quadratic_assumption} when $A = 0$. For $t \Gamma^{(Q)}_{s/2} - \frac{1}{c_0} \Gamma^{(Q)}_{t+1}$, write
\begin{align*}
    \psi^{(m)}_i C = \begin{bmatrix}
    \phi_i^1 C\\
    \phi_i^2 C\\
    \vdots\\
    \phi_i^k C
    \end{bmatrix}_{km \times d} = \begin{bmatrix}
    C & 0 & \dots & 0\\
    0 & C & \dots & 0\\
    \vdots\\
    0 & 0 & \dots & C
    \end{bmatrix}_{km \times kd}
    \begin{bmatrix}
    \phi_i^1 I_d \\
    \phi_i^2 I_d\\
    \vdots\\
    \phi_i^k I_d
    \end{bmatrix}_{kd \times d} = \mathbf{C} \psi_i^{(d)}.
\end{align*}
We have 
\begin{align*}
    \Gamma^{(Q)}_{t+1} = \mathbf{C} \Big[ \psi^{(d)}_1 Q (\psi^{(d)}_1)^\top + \dots + (\psi^{(d)}_{t} + \dots + \psi^{(d)}_1 A^{t-1}) Q (\psi^{(d)}_{t}  + \dots + \psi^{(d)}_1 A^{t-1})^\top \Big] \mathbf{C}^\top  
\end{align*}
By a similar argument and given assumption \eqref{eq:filter_quadratic_assumption}, we have $t \Gamma^{(Q)}_{s/2} - \frac{1}{c_0} \Gamma^{(Q)}_{t+1} \succeq 0$.
\end{proof}
\begin{remark} \normalfont When $A$ is symmetric ($A = U D U^\top$), the positive semi-definite condition filter quadratic function can be further simplified to $t \Omega_{s/2}(D; \psi) - \Omega_{t+1}(D; \psi) \succeq 0$ for all diagonal matrices $D$ with $|D_{ii}| \leq 1$.
\end{remark}

In the following lemma, we show a high probability upper bound on $\|Z_t^{-1/2} f_{t+1}\|_2$.  
\begin{lemm}\label{lemma:least_squares_second_term} \normalfont{\textbf{($\mathbf{\|Z_t^{-1/2} f_{t+1}\|_2}$ upper bound)}} \textit{Assume as in Lemma \ref{lemma:covariance_is_increasing} and let $\kappa$ be the maximum condition number of $Q$ and $R$. Define the following for all $t \geq 1$, regularization parameter $\alpha > 0$, $p = 3/20$, and fix $0 < \alpha_0 \leq 200 \alpha$ and $\delta > 0$}
\begin{align*}
    Z_t = \alpha I + \sum_{i=1}^t f_i f_i^\top, \quad
    \Gamma_{\max} = t [2km + 4 \log(4/\delta)][\alpha_0 I + \Gamma_t], \quad
    \Gamma_{\min} = \alpha I + \frac{s \lfloor t/s \rfloor p^2  \Gamma_{s/2}}{8}
\end{align*}
\textit{
For any $A$, suppose that there exists $t_0 \geq 1$ for which there exists $s$ such that }
\begin{align}\label{eq:filter_quadratic_assumption}
    s \leq \frac{tp^2/10}{\log \det(\Gamma_{\max}) - l \log (\alpha) + \log(4/\delta)}, \quad t \Omega_{s/2}(A; \psi) - \Omega_{t+1}(A; \psi) \succeq 0. \quad 
\end{align}
\textit{Then, for all $t \geq t_0$ with probability at least $1-\delta$}
\begin{align*}
    \|Z_{t-1}^{-1/2}f_t\|_2^2 \leq &  10\kappa (2mk + 4 \log(2/\delta))/p^2.
\end{align*}
\end{lemm}

\begin{proof}
Let $c_T = 10\kappa (2mk + 4 \log(2/\delta))/p^2 $. With probability at least $1-\delta$, we lower bound $\sum_{i=1}^t f_i f_i^\top$ by Lemma \ref{lemma:lower_bound_Z_t} and upper bound $\frac{1}{c_T} f_{t+1}f_{t+1}^\top$ by Lemma \ref{lemma:high_probability_upper_bound_ftftT}
\begin{align*}
    Z_t - \frac{1}{c_T} f_{t+1}f_{t+1}^\top & = \alpha I + \sum_{i=1}^t f_i f_i^\top - \frac{1}{c} f_{t+1}f_{t+1}^\top \\
    & \succeq \frac{\alpha}{2} I + \frac{p^2}{10} t \Gamma_{s/2} - \frac{p^2}{10}\alpha_0 I - \frac{p^2}{10} \frac{1}{c_0}\Gamma_{t+1}\\
    & \overset{\mathrm{(1)}}{\succeq} + \frac{p^2}{10} t \Gamma_{s/2} - \frac{p^2}{10} \frac{1}{c_0}\Gamma_{t+1}\\
    & \overset{\mathrm{(2)}}{\succeq} 0
\end{align*}
where inequality (1) is due to the assumption $\alpha_0 \leq 200 \alpha$ and (2) uses the result of Lemma \ref{lemma:filter_quadratic_condition}.

Using Schur complement lemma, $Z_t - \frac{1}{c_T} f_{t+1}f_{t+1}^\top$ is positive semi-definite if and only if the following matrix is positive semi-definite
\begin{align*}
    \begin{bmatrix}
    Z_t & f_{t+1}\\
    f_{t+1}^\top & c_T.
    \end{bmatrix}.
\end{align*}
Using the other Schur complement, this is true if and only if $c_T - f_{t+1}^\top Z^{-1}_t f_t \geq 0$. Equivalently,
\begin{align*}
    Z_t - \frac{1}{c_T} f_{t+1}f_{t+1}^\top \succeq 0 \quad \Leftrightarrow \quad \|Z_t^{-1/2} f_{t+1}\|_2 \leq c_T,
\end{align*}
which concludes the proof.
\end{proof}
The above lemma states that if $k \asymp_M \polylog(T)$ then $\|Z_{t-1}^{-1/2} f_t\|_2^2 \lesssim_M \polylog(T)$ with high probability.

\subsection{Proof of Lemma \ref{lemma:excitation_result}}\label{app:proof_lemma_1}
 
We now prove that $\|Z_{t-1}^{-1/2} f_t\|_2^2 \lesssim_M \polylog(T)$ implies $\sum_{t=1}^T \|Z_{t-1}^{-1/2} f_t\|_2^2 \lesssim_M \polylog(T)$. We first present a lemma inspired by Lemma 2 of \citet{lai1982least}.
\begin{lemm} \normalfont{\textbf{(Upper bound on $\mathbf{\sum_{i=1}^t \|Z_i^{-1/2} f_i\|_2^2}$)}}\label{lemma:poly_log_dependency_lai_lemma} \textit{Let $f_1, \dots, f_t$ be $l$-dimensional vectors and $Z_0$ an $l \times l$ positive definite matrix. Define $Z_t = Z_0 + \sum_{i=1}^t f_if_i^\top$. Then,}
\begin{align*}
    \sum_{i=1}^t f_i^\top  Z_i^{-1} f_i \leq \log\Big(\frac{\det(Z_t)}{\det(Z_0)}\Big).
\end{align*}
\end{lemm}
\begin{proof}
First, note that $Z_t$ is positive definite and has a positive determinant for all $t \geq 1$. Using matrix determinant lemma, we have 
\begin{align*}
    \det(Z_{t-1}) = \det(Z_t - f_tf_t^\top ) = \det(Z_t)(1-f_t^\top  Z^{-1}_t f_t) \Rightarrow f_t^\top   Z^{-1}_t f_t = \frac{\det(Z_t) - \det(Z_{t-1})}{\det(Z_t)}
\end{align*}
Since $Z_i \succeq Z_{i-1}$, we have $\det(Z_i) \geq \det(Z_{i-1})$. We write
\begin{align*}
    \sum_{i=1}^t f_i^\top  Z_i^{-1} f_i = \sum_{i=1}^t 1 - \frac{\det(Z_{i-1})}{\det(Z_i)} \leq \sum_{i=1}^t \log \Big( \frac{\det(Z_{i})}{\det(Z_{i-1})}\Big) = \log\Big(\frac{\det(Z_t)}{\det(Z_0)}\Big), 
\end{align*}
where we used the fact that $1-x \leq \log(1/x)$ for $x \leq 1$.
\end{proof}

We are now ready to prove Lemma \ref{lemma:excitation_result}.

\begin{prooflemma} The first claim is already proved in Lemma \ref{lemma:small_ball_condition_ft}. We focus on proving the second claim.  Recall the result of Lemma \ref{lemma:poly_log_dependency_lai_lemma}, which states that 
\begin{align*}
    \sum_{t=1}^T f_t^\top Z_t^{-1}f_t \leq \log \Big(\frac{\det(Z_T)}{\det(\alpha I)}\Big). 
\end{align*}
Using matrix determinant lemma, the above is equivalent to 
\begin{align*}
    \sum_{t=1}^T \frac{f_t^\top  Z_{t-1}^{-1} f_t}{1+f_t^\top  Z_{t-1}^{-1} f_t} \leq \log \Big(\frac{\det(Z_T)}{\det(\alpha I)}\Big).
\end{align*}
By Lemma \ref{lemma:high_probability_upper_bound_on_det_Z_t}, $\log \det(Z_t)$ is bounded by $\polylog(T)$ with high probability since $k \asymp_M \polylog(T)$. Furthermore, by Lemma \ref{lemma:least_squares_second_term}, $\|Z_{t-1}^{-1/2} f_t\|_2^2 \lesssim_M \polylog(T)$ with high probability. Concretely, 
\begin{align*}
    & \Prob(\|Z_{t-1}^{-1/2}f_t\|_2^2 \leq 10\kappa (2mk + 4 \log(2/\delta))/p^2) \geq 1 - \delta,\\
    & \Prob \Big(\log (\det(Z_t)) \leq mk \log \Big[ \alpha^2 + 8k (R_P^2+1)(R_x^2+R_C) (1+\gamma)^4 (mt - \log(\delta)) t^{3+2\log (\gamma)} \Big] \Big) \geq 1- \delta.
\end{align*}
Therefore, we can apply Lemma \ref{lemma:poly_log_sum} by combining the two bounds and taking a union bound
\begin{align*}
    & R_Z(T) \triangleq mk \log \Big[ \alpha^2 + 8k (R_P^2+1)(R_x^2+R_C) (1+\gamma)^4 (mT - \log(\delta)) T^{3+2\log (\gamma)} \Big],\\
    & \Prob \Big\{ \sum_{t=1}^T \|Z_{t-1}^{-1/2}f_t\|_2^2 \leq  \Big(1+\frac{10\kappa (2mk + 4 \log(4/\delta)}{p^2}\Big) \big(R_Z(T) - mk \log (\alpha) \big) \Big\} \geq 1-\delta.
\end{align*}
\end{prooflemma}
\subsection{Regularization term}\label{app:regularization_term}
The following lemma computes an upper bound on the 2-norm of the relaxed model parameters $\widetilde{\Theta}$.
\begin{lemm}\label{lemma:pca_parameter_bound} \normalfont{\textbf{(Model parameter bound)}} \textit{Consider system \eqref{eq:lds_definition} and let $k$ be the number of spectral filters and $\widetilde{\Theta}$ be the parameters defined in \eqref{eq:pca_parameters}. If $\|\cO_t\|_2, \|\mathcal{C}_t\|_2 \leq R_K$ and $\|D\|_2 \leq R_P$ then, 
\begin{align*}
    \|\widetilde{\Theta}\|_2 \leq 2kR_K + R_P.
\end{align*}}
\end{lemm}
\begin{proof}
Parameter matrix $\widetilde{\Theta}$ is the concatenation of coefficients of features $\widetilde{y}_{t-1}, \widetilde{x}_{t-1}, x_t$. By matrix norm properties,
\begin{align*}
    \|\widetilde{\Theta}\|_2 \leq \|D\|_2 + \sum_{j=1}^k \|\sum_{i=1}^d C v_i w_i^\top K \langle \mu(\lambda_i), \phi_j \rangle\|_2 + \|\sum_{i=1}^d C v_i w_i^\top (B-KD) \langle \mu(\lambda_i), \phi_j \rangle\|_2.
\end{align*}

Recall that $\{\lambda_i\}_{i=1}^k$, $\{v_i\}_{i=1}^k$, and $\{w_i^\top\}_{i=1}^k$ are the top $k$ eigenvalues, right eigenvectors, and left eigenvectors of $G$, respectively. Write
\begin{align*}
    \|\sum_{i=1}^d C v_i w_i^\top K \langle \mu(\lambda_i), \phi_j \rangle\|_2
    = \|\sum_{t=1}^{T} C G^{T-t} K \phi_j(t)\|_2
    = \|\cO_T \phi_j\|_2
     \leq R_{K},
\end{align*}
and similarly,
\begin{align*}
    \|\sum_{i=1}^d C v_i w_i^\top (B-KD) \langle \mu(\lambda_i), \phi_j \rangle\|_2 
     = \|\mathcal{C}_T \phi_j\|_2
     \leq R_{K}.
\end{align*}
Summing all terms gives the final bound.
\end{proof}

\begin{lemm} \label{lemma:regularization_term_bound} \normalfont{\textbf{(Regularization term bound)}} \textit{Assume as in Lemma \ref{lemma:pca_parameter_bound} and let $Z_t = \alpha I + f_t f_t^\top$. If $\alpha \leq 1/\|\widetilde{\Theta}\|_2^2$, then
\begin{align*}
    \|\alpha \widetilde{\Theta} Z_{t-1}^{-1/2}\|_2^2 \leq 1.
\end{align*}}
\end{lemm}
\begin{proof}
The regularization term implies $Z_t \succeq \alpha I$ and thus $\|Z_t^{-1/2}\|_2^2 \leq 1/\alpha$. By norm properties
\begin{align*}
    \|\alpha \widetilde{\Theta} Z_{t-1}^{-1/2}\|_2^2 \leq \alpha^2 \|\widetilde{\Theta}\|_2^2 \|Z_{t-1}^{-1/2}\|_2^2 \leq 1.
\end{align*}
\end{proof}

\subsection{Innovation error}\label{app:innovation_error}
The following lemma, based on the analysis given by \citet{tsiamis2020online}, shows that the innovation error is bounded by $\sqrt{\cL(T)}$ (defined in \eqref{eq:risk_definition_L_t}).

\begin{lemm}\label{lemma:innovation_error} \normalfont{\textbf{(Innovation error bound)}} \textit{ Let $\cL(T) = \sum_{t=1}^T \|\hat{m}_t - m_t\|_2^2$ be the squared error between Kalman predictions in hindsight and predictions by Algorithm \ref{alg:slip}. Assume that the innovation covariance matrix has a bounded norm $\|V\|_2 \leq R_V$. For all $\delta > 0$, the following holds with probability greater than $1-\delta$:}
\begin{align*}
       \sum_{t=1}^T 2 e_t^\top (\hat{m}_t - m_t) \leq 8R_V^2 \Big(\cL(T) + 1\Big)^{1/2} \Big[ 2 + \log \Big(\frac{\cL(T) + 1}{\delta}\Big) \Big].
\end{align*}
\end{lemm}
\begin{proof}
Write 
\begin{align*}
    \sum_{t=1}^T e_t^\top (\hat{m}_t - m_t) = \sum_{t=1}^T \sum_{i=1}^m e_{t,i} (\hat{m}_{t,i} - m_{t,i}).
\end{align*}
Let $s = m \lfloor s/m \rfloor + r$ and define the following filtration 
\begin{align*}
    \mathcal{F}_s = \{e_{1,1}, \dots, e_{\lfloor s/m \rfloor, r}\}.
\end{align*}
A scalar version of Theorem \ref{thm:self_normalized_vector_matringale} states that the following holds with probability at least $1-\delta$
\begin{align*}
    \Big( \sum_{t=1}^T \|\hat{m}_t - m_t\|_2^2 + 1\Big)^{-1/2} \sum_{t=1}^T e_t^\top (\hat{m}_t - m_t) \leq 4R_V^2 \Big[ 2 + \log \Big( \frac{1}{\delta} \Big) + \log \Big( \sum_{t=1}^T \|\hat{m}_t - m_t\|_2^2 + 1\Big)\Big].
\end{align*}
Therefore, with probability at least $1-\delta$
\begin{align*}
       \sum_{t=1}^T 2 e_t^\top (\hat{m}_t - m_t) \leq 8R_V^2 \Big(\cL(T) + 1\Big)^{1/2} \Big[ 2 + \log \Big(\frac{\cL(T) + 1}{\delta}\Big) \Big].
\end{align*}
\end{proof}

\subsection{Proof of Theorem 1}\label{app:proof_thm_1}

\noindent \begin{proofregret} Recall the regret decomposition given in Appendix \ref{app:regret_decomposition}:
\begin{alignat*}{2}
    \text{Regret}(T) \leq & \sup_{1 \leq t \leq T} \Big(\|E_{t-1} Z_{t-1}^{-1/2}\|_2^2 +  \|B_{t-1} Z_{t-1}^{-1/2}\|_2^2 + \|\alpha \widetilde{\Theta} Z_{t-1}^{-1/2}\|_2^2 \Big)\Big(\sum_{t=1}^T  \|Z_{t-1}^{-1/2} f_t\|_2^2\Big)\\
    & + T \sup_{1 \leq t \leq T} \|b_t\|_2^2 - \sum_{t=1}^T 2e_t^\top (\hat{m}_t - m_t).
\end{alignat*}

Let $\delta_1 = \delta/8$. We describe bounds on each term in the above regret bound. All lemmas and theorems used in this proof contain explicit dependencies on horizon $T$ as well as PAC bound parameters. While one can combine these results to write a regret bound with explicit dependencies on all parameters, we refrain from writing in such detail here for a clear presentation.

\paragraph{Bounding $\mathbf{\|E_{t-1} Z_{t-1}^{-1/2}\|_2^2}$.} According to Theorem \ref{thm:self_normalized_vector_matringale}, with probability at least $1-\delta_1$, the term $\|E_{t-1} Z_{t-1}^{-1/2}\|_2^2$ is bounded by 
    \begin{align*}
        \|E_{t-1} Z_{t-1}^{-1/2}\|_2 \lesssim \poly(R_\Theta, m) \Big[ \log(1/\delta_1) +
        \log(\det(Z_t)) - l\log(\alpha)\Big],
    \end{align*}
    $l = (m+n)k+n$ is the feature vector dimension. We substitute the regularization parameter $\alpha$ and the number of filters $k$ according to  Theorem \ref{thm:regret_bound_paper} assumption (iii). Given the values for $k, \alpha$ and by Lemma \ref{lemma:high_probability_upper_bound_on_det_Z_t}, with probability at least $1-\delta_1$ we have 
    \begin{align*}
        \log(\det(Z_t)) \lesssim \poly(R_\Theta, m, \beta) \polylog(\gamma, \frac{1}{\delta_1}) \log^3(T).
    \end{align*}
    Taking a union bound gives
    \begin{align}\label{eq:proof_first_term}
        \Prob \Big[ \|E_{t-1} Z_{t-1}^{-1/2}\|_2^2 \lesssim \poly(R_\Theta, m, \beta) \polylog(\gamma, \frac{1}{\delta_1}) \log^6(T) \Big] \geq 1-2\delta_1.
    \end{align}
    
\paragraph{Bounding $\mathbf{\|B_{t-1} Z_{t-1}^{-1/2}\|_2^2}$.} Recall the definitions $B_t = \sum_{i=1}^t b_i f_i^\top$ from \eqref{eq:definitions_of_Z_E_B} and $b_i = \widetilde{\Theta}f_t - m_t$ from \eqref{eq:innovation_bias_definitions}. We choose the number of filters $k$ to satisfy \eqref{eq:convex_relaxation_assumption_on_k} with failure probability $\delta_1 > 0$ and $\epsilon = 1/T$,\footnote{Setting $\epsilon = 1/T$ is later used for a uniform bound on $\|b_t\|_2^2$ and is not critical in this part of the proof.} which results in $k \gtrsim_M \log^2(T)$ satisfied by assumption (iii). Therefore, we can apply Theorem \ref{thm:convex_relaxation_bound} which states that $\|b_t\|_2^2 \leq 1/T$ with probability at least $1-\delta_1$. Combining this result with the result of Theorem \ref{thm:self_normalized_vector_matringale} with a union bound yields
\begin{align}\notag 
    \Prob \Big[ \|\Big(\sum_{i=1}^{t-1} b_i f_i^\top  \Big)Z_{t-1}^{-1/2} \|_2 \leq \frac{4}{T}(2m+ \log \Big( \frac{\det(Z_t)^{1/2} \det(\alpha I_l )^{-1/2}}{\delta_1}\Big) )\Big] \geq 1-2\delta_1.
\end{align}
With a similar argument used in bounding $\|E_{t-1} Z_{t-1}^{-1/2}\|_2^2$, we have
\begin{align}\label{eq:proof_second_term}
    \Prob \Big[ \|B_{t-1} Z_{t-1}^{-1/2}\|_2^2 \lesssim \poly(R_\Theta, m, \beta) \polylog(\gamma, \frac{1}{\delta_1}) \frac{\log^6(T)}{T} \Big] \geq 1-3\delta_1.
\end{align}

\paragraph{Bounding $\mathbf{\|\alpha \widetilde{\Theta} Z_{t-1}^{-1/2}\|_2^2}$.} By assumption (iii) and as a result of Lemma \ref{lemma:regularization_term_bound}, we have 
\begin{align*}
    \|\alpha \widetilde{\Theta} Z_{t-1}^{-1/2}\|_2^2 \lesssim 1.
\end{align*}

\paragraph{Bounding $\mathbf{\sum_{t=1}^T  \|Z_{t-1}^{-1/2} f_t\|_2^2}$.} Lemma \ref{lemma:excitation_result} provides the following bound on the excitation term
\begin{align}\label{eq:proof_third_term}
    \Prob \Big[ \sum_{t=1}^T \|Z_{t-1}^{-1/2}f_t\|_2^2 \lesssim \kappa \poly(R_\Theta, m, \beta) \polylog(\gamma, \frac{1}{\delta_1}) \log^5(T) \Big] \geq 1-\delta_1,
\end{align}
where the number filters $k$ is substituted by assumption (iii). 

\paragraph{Bounding $\mathbf{T \sup_{1 \leq t \leq T} \|b_t\|_2^2}$.} Applying Theorem \ref{thm:convex_relaxation_bound} with parameters $\delta_1 > 0, \epsilon = 1/T$, we have
\begin{align}\label{eq:proof_fourth_term}
    \Prob \big[ T \sup_{1 \leq t \leq T} \|b_t\|_2^2 \leq T \epsilon \leq 1 \big] \geq 1-\delta_1.
\end{align}

Recall from Appendix \ref{app:regret_decomposition} that $\cL(T)$ is bounded by 
\begin{align*}
    \cL(T) \leq \sup_{1 \leq t \leq T} \Big(\|E_{t-1} Z_{t-1}^{-1/2}\|_2^2 +  \|B_{t-1} Z_{t-1}^{-1/2}\|_2^2 + \|\alpha \widetilde{\Theta} Z_{t-1}^{-1/2}\|_2^2 \Big)\Big(\sum_{t=1}^T  \|Z_{t-1}^{-1/2} f_t\|_2^2\Big) + T \sup_{1 \leq t \leq T} \|b_t\|_2^2.
\end{align*}
Lemma \ref{lemma:innovation_error} with $\delta_1$ states that
\begin{align}\label{eq:proof_fifth_term}
    \Prob \Big[ \sum_{t=1}^T e_t^\top (\hat{m}_t - m_t) \lesssim \poly(R_\Theta) \polylog\Big(\frac{1}{\delta_1}\Big) \sqrt{\cL(T)+1} \Big] \geq 1-\delta_1.
\end{align}

Combining the bounds given in \eqref{eq:proof_first_term}, \eqref{eq:proof_second_term}, \eqref{eq:proof_third_term}, \eqref{eq:proof_fourth_term}, \eqref{eq:proof_fifth_term}, taking a union probability bound, and setting $\delta = 8\delta_1$ gives 
\begin{align*}
    \Prob \Big[ \text{Regret}(T) \leq \kappa \log^{11}(T) \poly(R_\Theta, \beta, m) \polylog(\gamma, \frac{1}{\delta}) \Big] \geq 1- \delta.
\end{align*}
\end{proofregret}

\section{Auxiliary lemmas}\label{app:technical_lemmas}
In this section, we present a few lemmas that we use throughout the theoretical analysis of our algorithm, presented here for completeness.

The following lemma provides an upper bound on the norm of block Toeplitz matrices \Citep{tsiamis2019finite}.
\begin{lemm}\label{lemma:triangular_block_toeplitz_2_norm}\normalfont{\textbf{(Triangular Block Toeplitz Norm)}} \textit{Let $\cT_i \in \mathbb{R}^{m_1, m_2}$ for $i = 1, 2, \dots, n$. Define the following triangular block Toeplitz matrix}
\begin{align*}
    \cT = \begin{bmatrix}
    \cT_1 & \cT_2 & \cT_3 & \dots & \cT_{n-1} & \cT_n\\
    0 & \cT_1 & \cT_2 & \dots & \cT_{n-2} & \cT_{n-1}\\
    \vdots\\
    0 & 0 & 0 & \dots & \cT_1 & \cT_2\\
    0 & 0 & 0 & \dots & 0 & \cT_1
    \end{bmatrix}.
\end{align*}
\textit{Then,}
\begin{align*}
    \|\cT\|_2 \leq \sum_{i=1}^n \|\cT_i\|_2.
\end{align*}
\end{lemm}

The following is a simple result for upper bounding a series.
\begin{lemm} \label{lemma:poly_log_sum}\textit{Let $t \in \mathbb{N}$ and let $z_t$ to be a non-negative sequence bounded by a non-decreasing poly-logarithmic function $g(t)$. Suppose that the following sum
\begin{align*}
    \sum_{t=1}^T \frac{z_t}{1+z_t}
\end{align*}
is bounded by $h(T)$, a non-decreasing poly-logarithmic function of $T$. Then, $\sum_{t=1}^T z_t$ is bounded by a non-decreasing function poly-logarithmic in $T$.}
\end{lemm}
\begin{proof}
Let $z_m = \max\limits_{t \in \{1,\dots, T \}} z_i$. We have $z_m \leq g(m) \leq g(T)$. Therefore,
\begin{align}\label{eq:combined_bounds_series}
    \sum_{t=1}^T z_t \leq \sum_{t=1}^T \frac{1+z_m}{1+z_t} z_t \leq (1+g(T)) \sum_{t=1}^T \frac{z_t}{1+z_t} \leq (1+g(T)) h(T),
\end{align}
which is the desired conclusion.
\end{proof}

\section{Additional experiments}\label{app:extra_experiments}

\paragraph{Comparison with the EM algorithm.} We conduct an experiment in a scalar LDS to compare the performance of our algorithm with the EM algorithm that estimates system parameters (Figure \ref{fig:experiments_extra}, left). The parameters estimated by the EM algorithm are later used by the Kalman filter for predictions. In this experiment, we set the horizon $T = 200$ due to the large computation time required by the EM algorithm. The number of filters $k$ is set to 5 for all other three algorithms. The experiment was simulated 100 independent times and the average error together with the 99\% confidence intervals are presented.

\begin{figure}[h]
    \centering
    \includegraphics[scale = 0.33]{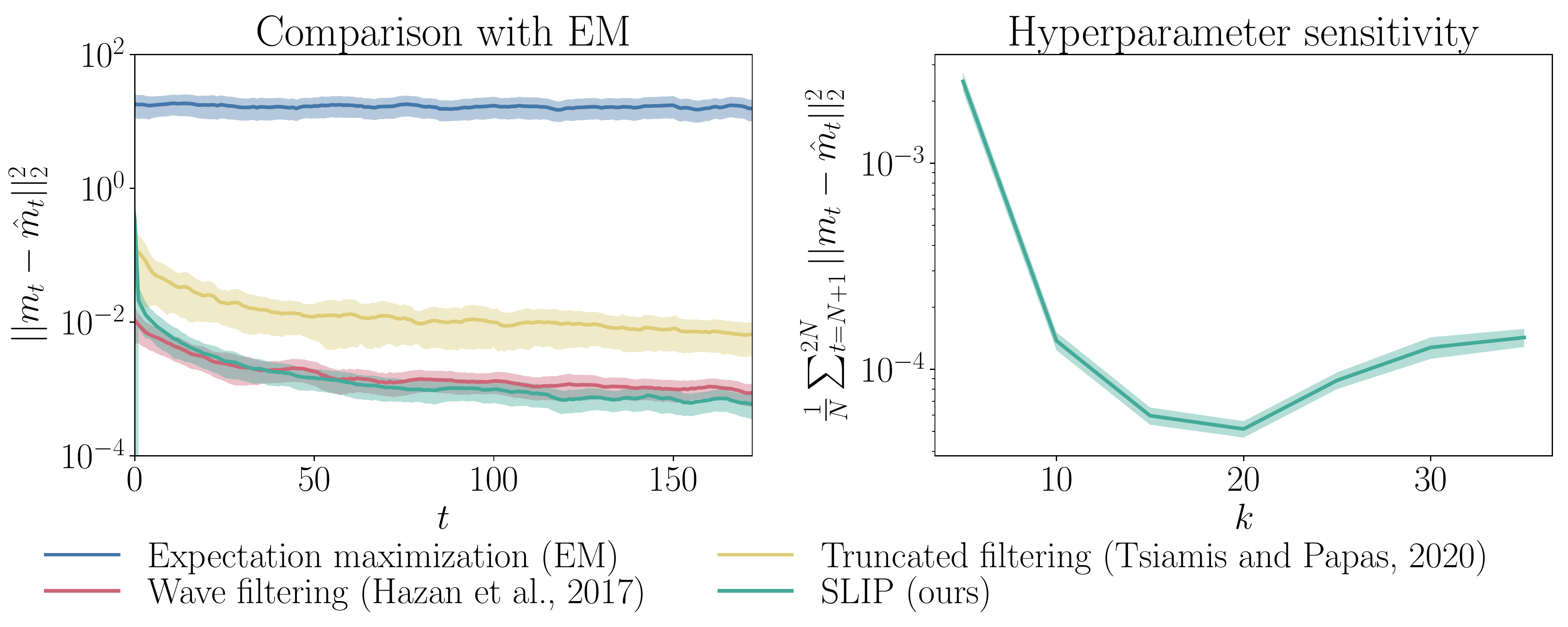}
    \caption{Left: Performance of our algorithm compared with wave filtering, truncated filtering, and expectation maximization in a scalar system with parameters $A = B = C = D = 1$, noise covariance matrices $Q = R = 0.001$, inputs $x_t \sim \mathcal{N}(0,2)$, and horizon $T = 200$. Right: Hyperparameter sensitivity of our algorithm in the same systems with inputs $x_t \sim \mathcal{N}(0,0.5)$ and horizon $T = 10000$.}
    \label{fig:experiments_extra}
\end{figure}

For the system considered in this experiment, EM performs poorly. System-identification-based methods such as EM, besides being significantly slower, do not have regret guarantees and they can fail in some examples; a similar observation was made by \Citet{hazan2017learning}.

\paragraph{On hyperparameters.} The SLIP algorithm has two hyperparameters: the number of filters $k$ and the regularization parameter $\alpha$. In the experiments, we set $\alpha > 0$ only when the empirical feature covariance matrix is singular, which we observe only happens in the first two time steps. For the number of filters $k$, Theorem \ref{thm:regret_bound_paper} provides a guideline of choosing $k$ of order $\log^2(T)$. The right plot in Figure \ref{fig:experiments_extra} demonstrates the sensitivity of the SLIP algorithm with respect to the number of filters $k$. The system considered for this experiment is scalar with Gaussian inputs and the horizon is set to 10000. As before, the experiment was simulated 100 independent times. We vary $k$ from 5 to 35 and measure the average prediction error from 5000 to 10000 ($N = 5000$ in the plot). We observe that the SLIP algorithm is robust with respect to parameter $k$.

\end{document}